\documentclass[10pt]{article} 
\usepackage[preprint]{tmlr}

\usepackage{amsmath,amsfonts,bm}

\def\eqref#1{equation~\ref{#1}}

\def\1{\bm{1}}

\DeclareMathAlphabet{\mathsfit}{\encodingdefault}{\sfdefault}{m}{sl}
\SetMathAlphabet{\mathsfit}{bold}{\encodingdefault}{\sfdefault}{bx}{n}

\usepackage{hyperref}
\usepackage{url}
\usepackage{booktabs}

\newcommand{\ours}{\textsc{Ours}}
\usepackage{multirow}
\usepackage{floatrow}
\usepackage{amsthm}\usepackage{amssymb}
\usepackage[english]{babel}
\newtheorem{theorem}{Theorem}
\newtheorem{proposition}{Proposition}
\newtheorem{remark}{Remark}
\usepackage{tcolorbox}
\newcommand{\probP}{\mathbb{P}}

\usepackage{graphicx}\usepackage[x11names,dvipsnames]{xcolor}

\usepackage{amsmath,amssymb,amsfonts}
\usepackage{algcompatible}
\usepackage{graphicx}
\usepackage{textcomp}
\usepackage{xcolor}

\usepackage{algorithm}        
\usepackage{xcolor}          
\usepackage{textcomp}        
\usepackage{floatrow}
\usepackage{amsthm}
\newtheorem{definition}{Definition}

\author{Praneeth Vepakomma $^*$
}
\author{Kaustubh Ponkshe\ast}

\def\month{10}  
\def\year{2023} 
\def\openreview{\url{https://openreview.net/forum?id=XXXX}} 

\usepackage{graphicx}
\usepackage{indentfirst,csquotes}

\usepackage{amssymb,amsthm,amsmath}
\usepackage{xcolor,paralist,hyperref,titlesec,fancyhdr,etoolbox}

\newtheorem{lemma}[theorem]{Lemma}

\hypersetup{ colorlinks=true, linkcolor=black, filecolor=black, urlcolor=black }

\usepackage{lipsum}
\usepackage{abstract}
\usepackage{authblk}

\title{Power Mechanism: Private Tabular Representation Release
for Model Agnostic Consumption}

\author{\name Praneeth Vepakomma \email vepakom@mit.edu \\
      \addr MBZUAI, MIT
      \AND
      \name Kaustubh Ponkshe \email kaustubh.ponkshe@epfl.ch \\
      \addr EPFL
      }

\def\month{MM}  
\def\year{YYYY} 
\def\openreview{\url{https://openreview.net/forum?id=XXXX}} 

\begin{document}

\maketitle
\def\thefootnote{*}\footnotetext{These authors contributed equally to this work}\def\thefootnote{\arabic{footnote}}
Authors contributed equally. Order determined by coin flip.\footnote{normal footnote}
\begin{abstract}Traditional collaborative learning approaches are based on sharing of model weights between clients and a server. However, there are advantages to resource efficiency through schemes based on sharing of embeddings (activations) created from the data. Several differentially private methods were developed for sharing of weights while such mechanisms do not exist so far for sharing of embeddings. We propose {\ours} to learn a privacy encoding network in conjunction with a small utility generation network such that the final embeddings generated from it are equipped with formal differential privacy guarantees. These privatized embeddings are then shared with a more powerful server, that learns a post-processing that results in a higher accuracy for machine learning tasks. We show that our co-design of collaborative and private learning results in requiring only one round of privatized communication and lesser compute on the client than traditional methods. The privatized embeddings that we share from the client are agnostic to the type of model (deep learning, random forests or XGBoost) used on the server in order to process these activations to complete a task.
\end{abstract} 

\section{Introduction}Modern privacy-preserving machine learning methods, exemplified by approaches like DP-SGD \cite{abadi2016deep}, focus on protecting privacy by adding noise to model weights during training. However, in many real-world scenarios, organizations need to share intermediate data representations (activations or embeddings) rather than model weights and yet no formal privacy guarantees exist for such sharing. This work addresses this gap by introducing a mechanism for clients to privately share data embeddings as opposed to model weights while maintaining formal differential privacy guarantees \cite{dwork2008differential,dwork2014algorithmic}. This flexibility is in contrast to existing private federated learning approaches \cite{bhowmick2018protection}, which require clients to train complete models locally and share privatized weights. The proposed approach addresses privacy challenges in split learning architectures while preserving their benefits in computational efficiency. Specifically, we develop a principled approach for private activation sharing through a careful co-design of collaborative and private learning mechanisms. This design provides formal differential privacy guarantees while enabling efficient distributed computation.
\subsection{Approach}To detail our approach, we begin with the most basic scenario in privacy-preserving distributed learning: a single client seeking to securely offload the majority of training computation to a server while maintaining data privacy. In this context, we propose Power Learning, a mechanism that creates privacy-preserving embeddings. The figure \ref{mainFig} shows the high level functioning of our method 
\begin{figure*}[!htbp]
    \centering
\includegraphics[scale=0.48]{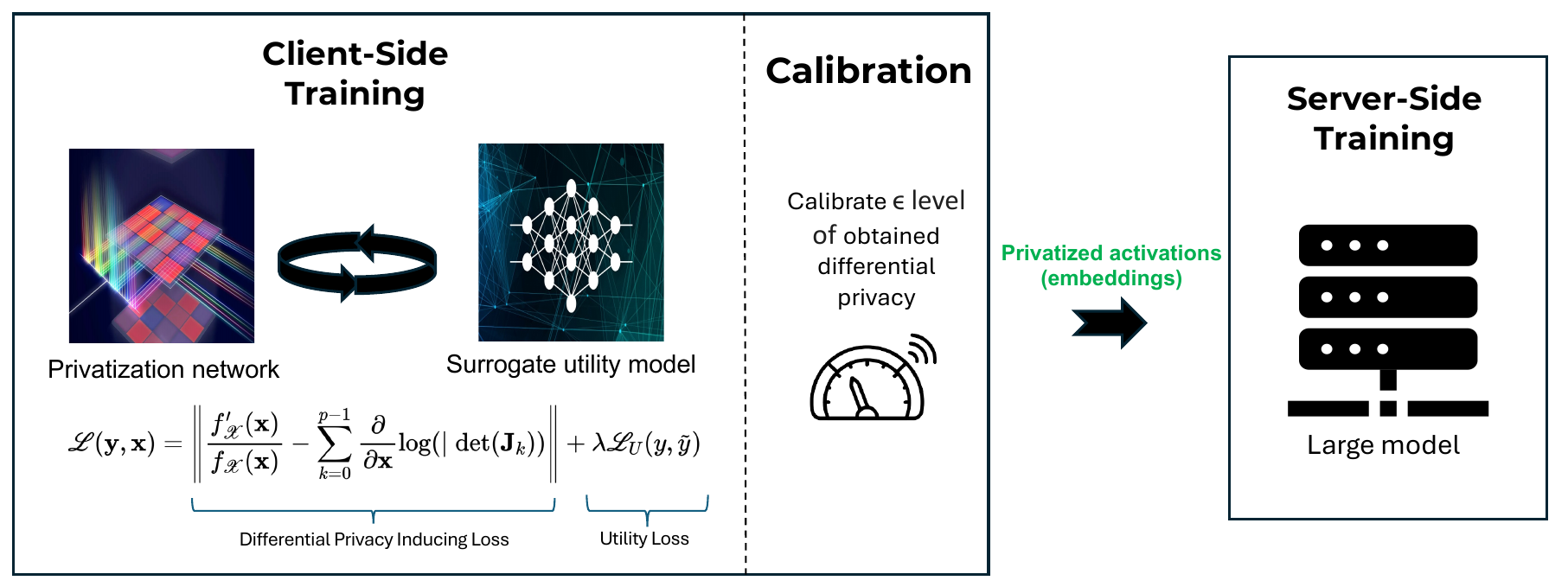}
    \centering
    \caption{Schematic illustration of the {\ours} for distributed and private learning, that theoretically calibrates and measures the obtained level of $\epsilon$ and $\delta$ for differential privacy. This calibration is done after the minimization of a specifically proposed privacy loss that is minimized in regularization with the machine learning utility loss.}
    \label{mainFig}
\end{figure*}
 based on Lipschitz privacy \cite{koufogiannis2015gradual} that enables us to quantify the privacy loss of data transformations through the properties of their gradients. This mathematical connection allows us to formulate privacy preservation as a differentiable loss. By incorporating this privacy loss as a regularizer alongside the standard utility objective, we create a joint optimization framework that simultaneously ensures both the privacy and utility of the generated embeddings. This co-optimization approach stands in contrast to previous methods that typically handle privacy and utility separately.  
Through extensive empirical evaluation, we demonstrate that our approach effectively prevents feature space hijacking attacks \cite{pasquini2021unleashing}, a significant vulnerability in traditional split learning systems. A key innovation of our method is its ability to calibrate privacy levels for individual samples during the embedding generation process, allowing for fine-grained privacy control based on data sensitivity. The framework achieves these privacy guarantees while maintaining high resource efficiency. Unlike existing approaches that require multiple rounds of communication, our method requires only a single round of communication with a compact payload, significantly reducing the client's computational and communication overhead. This efficiency is achieved without compromising the privacy guarantees or utility of the learned representations.

\subsubsection{Privatized Tabular Data Sharing}
Our approach enables clients with sensitive tabular data to create private embeddings that can be safely shared with a more computationally powerful server. A key advantage of our method is that these privatized embeddings are model-agnostic and the server can process them using any standard machine learning approach, from neural networks to decision trees, XGBoost, or generalized linear models. 

\subsection{Contributions}
\begin{enumerate}
    \item \textbf{Private Activation Sharing Framework} We introduce the first framework for sharing neural network activations with formal differential privacy guarantees. Unlike existing approaches that focus on privatizing model weights, our method enables clients to share private data representations while maintaining both privacy and utility. Our framework is model-agnostic, allowing servers to employ any machine learning approach (neural networks, random forests, XGBoost) for downstream tasks without modification.
    
    \item \textbf{Privacy Guarantees} We develop a novel regularized learning scheme that provides theoretical privacy guarantees for the shared activations. Our approach introduces a privacy-inducing loss function that guides the learning of private embeddings and provides a rigorous method to quantify the achieved privacy level post-optimization. We establish formal $(\epsilon,\delta)$-differential privacy guarantees through a novel theoretical analysis that includes uncertainty quantification of the privacy bounds.
    
    \item \textbf{Resource-Efficiency} We demonstrate significant improvements in computational and communication efficiency over existing private learning approaches. Our method requires only one round of client-server communication and reduces client-side computation through efficient embedding generation. Through extensive empirical evaluation, we show that our approach achieves better privacy-utility trade-offs compared to state-of-the-art methods while maintaining minimal computational overhead on the client side.
\end{enumerate}

\section{Related Work}\label{relatedwork}
 In this section, we categorize several related works and then compare them with power learning on several criteria as summarized in table \ref{tabComp}.


\begin{table}[!htbp]
\centering
\caption{This table contrasts the differences in the privacy problem catered to by the methods described in the related work in comparison with that of power learning. The unit of privacy in all is Training data}
\begin{center}\label{tabComp}
\[\begin{array}{|c|c|c|c|}
\hline \textbf { Method } & \begin{array}{c}
\textbf { Query Output } \\
\textbf { (Privatized) }
\end{array} & \begin{array}{c}
\textbf { What is not } \\
\textbf { private? }
\end{array} & \begin{array}{c}
\textbf { Post-Processing} \\
\textbf { Output }
\end{array} \\
\hline \begin{array}{c}\text {DP-SGD} \end{array} & \text { Model Weights } & \begin{array}{c}
\text { Training } \\
\text { activations }
\end{array} & \text { Predictions } \\
\hline \text { PATE } & \begin{array}{c}
\text { Test } \\
\text { Predictions }
\end{array} & \begin{array}{c}
\text { Teacher } \\
\text { Models }
\end{array} & \text { Student Model } \\
\hline \begin{array}{c}
\text { Power } \\
\text{Learning}
\end{array} & \begin{array}{c}
\text { Client Model } \\
\text { Embeddings }
\end{array} & \begin{array}{c}
\text { Client } \\
\text { Model }
\end{array} & \text { Server Model } \\
\hline

\end{array} \] \end{center}
\label{tab-rel}
\end{table}

\subsection{DP-SGD, PATE, Federate Learning and variants}The method of DP-SGD introduced in \cite{abadi2016deep} modifies stochastic gradient descent (SGD) based optimization used in learning neural networks by clipping the gradient for each lot of data and adding Gaussian noise to it. This approach is calibrated to ensure differential privacy of the learnt model with respect to the training data.  An improvement of this could be to perform DP-SGD with privacy amplification methods based on sampling or shuffling. Such amplification methods are not specific to DP-SGD and can be applied across the board to several kinds of differentially private mechanisms.
A recent alternative that improves over methods of DP-SGD with privacy amplification was that of DP-FTRL. This provides a better privacy-utility tradeoff while not necessitating any amplification. The basic idea is inspired by the binary-tree mechanism for differential privacy in computing private prefix sums. This helps with adding much lesser noise for release of gradient sums as the sum operation can be performed with a sensitivity proportional to only a log factor as the  query (optimization update in FTRL or follow-the-regularized-leader) can be reduced to private prefix sum computations. The Private Aggregation of Teacher Ensembles (PATE) framework was introduced to scale up training of differentially private models to large datasets and complex models. In this framework, multiple teacher models that are trained on disjoint subsets of sensitive data, guide a student model through knowledge transfer, while differential privacy is maintained via a noisy aggregation mechanism. However, some works \cite{tramèr2021differentiallyprivatelearningneeds} have shown that linear models trained on handcrafted features significantly outperform end-to-end deep neural networks for moderate privacy budgets . \\

We now compare these works in table \ref{tab-rel} with respect to criteria including the type of query that is privatized, the unit with respect to which the privacy is provided and the post-processing performed in order to attain the needed utility. PATE privatizes the predictions while power mechanism privatizes client model embeddings (which are not in label space but in real-valued embedding space). The client model embeddings are post-processed at the server to then obtain the predictions. The main benefit is a.) resource-efficiency gain for the client in this client-server setting, b.) private release in embedding space given the generative AI world we are moving into.

\section{Preliminaries}
We summarize the main notation used in this paper in Table~\ref{notTab} given below.

here are several formally established notions of privacy such as pure $\epsilon$-differential privacy and approximate $(\epsilon,\delta)$-differential privacy \cite{dwork2008differential,dwork2010differential,dwork2011differential,dwork2014algorithmic}.
 Other mathematical notions of privacy that have equivalences with differential privacy include Lipschitz privacy \cite{koufogiannis2015gradual,koufogiannis2015optimality,chatzikokolakis2013broadening,koufogiannis2017privacy,koufogiannis2016location} that is based on a Lipschitz requirement over the log density of the output of the query operating on sensitive data. Other equivalences include Blowfish privacy \cite{he2014blowfish,nie2010performance,machanavajjhala2015designing} and Pufferfish privacy \cite{kifer2014pufferfish,song2017pufferfish,kifer2012rigorous} which allows the user to
specify a class of protected predicates that must be learned subject to the guarantees of differential privacy, and all other predicates can be learned without differential privacy. However, differential privacy is
conservative and adversaries may not be able to leak as much
information as suggested by the theoretical bound \cite{nasr2021adversaryinstantiationlowerbounds}.The variants of zero-concentrated differential privacy (zCDP) \cite{dwork2016concentrated}, Renyi differential privacy (RDP) \cite{mironov2017renyi} and f-differential privacy (f-DP) \cite{dong2019gaussian} were introduced to avoid overly conservative budgeting of the obtained privacy level. Such a budgeting thereby helps improve the trade-off between the utility in answering queries properly and the achieved level of privacy. Each of the above notions of privacy has a formal mathematical definition of privacy as opposed to being a heuristic. There is a lengthy body of work of several privacy-preserving mechanisms that can help attain one or more of these notions of privacy, for various queries and at times with equivalences to pure and approximate differential privacy.

\begin{table}[!htbp]
\centering
\begin{tabular}{|l|l|}
\hline
{Population}                        & $\mathcal{X}$                                   \\ \hline
Input sample                        & $\mathbf{x} \in \mathbb{R}^d$                  \\ \hline
Input dataset                       & $\mathbf{X} \in \mathbb{R}^{n \times d}$       \\ \hline
Invertible and differentiable transformations & 
$\begin{aligned}
    &g_0: \mathbf{x} \mapsto \mathbf{w}_1, \\
    &g_i: \mathbf{w}_i \mapsto \mathbf{w}_{i+1} \text{ for } i \in \{1,\ldots, p-2\}, \\
    &g_{p-1}: \mathbf{w}_{p-1} \mapsto \mathbf{z}
\end{aligned}$ \\ \hline
Intermediate and final activations of privacy network & $\mathbf{w}_i \in \mathbb{R}^d, \mathbf{z} \in \mathbb{R}^d$ \\ \hline
Transformed dataset            & $\mathbf{Z} \in \mathbb{R}^{n \times d}$          \\ \hline
Composition                  & 
$\begin{aligned}
    &G: \mathcal{X} \to \mathcal{Z} \text{ where } \\
    &G(\mathbf{x}) = g_{p-1} \circ \cdots \circ g_0(\mathbf{x}) = \mathbf{z}
\end{aligned}$ \\ \hline
Intermediate transformed sample & $\mathbf{z}^i \in \mathbb{R}^d$ \\ \hline
Distribution of Population     & $f_{\mathcal{X}}(\mathbf{x})$     \\ \hline
Gradient of the distribution input sample & $\nabla_{\mathbf{x}} f_{\mathcal{X}}(\mathbf{x})$ \\ \hline
Jacobian of transform          & $\mathbf{J}_k = \frac{\partial \mathbf{w}_k}{\partial \mathbf{w}_{k-1}}$ \\ \hline
Privacy Network                & $\mathbf{P}_{\!N}$                \\ \hline
Utility Network                & $\mathbf{U}_{\!N}$                \\ \hline
True label of the input data   & $\mathbf{y}$                      \\ \hline
Label predicted by the model   & $\mathbf{\tilde{y}} = \mathbf{U}_{\!N}(\mathbf{z})$ \\ \hline
Privacy Loss                   & $\mathcal{L}_P(\mathbf{z},\mathbf{x})$ \\ \hline
Utility Loss                   & $\mathcal{L}_U(\mathbf{y},\mathbf{\tilde{y}})$ \\ \hline
\end{tabular}
\caption{List of main notations used in this paper.}
\label{notTab}
\end{table}

\subsection{Differential Privacy}\label{sec:dp}

\begin{definition}[$\epsilon$-Differential Privacy \cite{dwork2014algorithmic}]
A randomized algorithm $\mathcal{M}\colon \mathcal{X} \to \mathcal{Z}$ is $\epsilon$-differentially private if, for all neighboring datasets $\mathbf{X}, \mathbf{X'} \in \mathcal{X}$ and all $Z \subseteq \mathcal{Z}$,
\[
\Pr[\mathcal{M}(\mathbf{X}) \in Z] \leq e^\epsilon \Pr[\mathcal{M}(\mathbf{X'}) \in Z].
\]
\end{definition}

The notion of neighboring datasets differing in one record uses the Hamming metric. Other versions of differential privacy may use different neighborhood metrics, such as in metric differential privacy.

\begin{definition}[$(\epsilon,\delta)$-Differential Privacy]
A randomized algorithm $\mathcal{M}\colon \mathcal{X} \to \mathcal{Z}$ is $(\epsilon,\delta)$-differentially private if, for all neighboring datasets $\mathbf{X}, \mathbf{X'} \in \mathcal{X}$ and all $Z \subseteq \mathcal{Z}$,
\[
\Pr[\mathcal{M}(\mathbf{X}) \in Z] \leq e^\epsilon \Pr[\mathcal{M}(\mathbf{X'}) \in Z] + \delta.
\]
\end{definition}

\subsection{Lipschitz Privacy}
We use an equivalent notion of differential privacy called Lipschitz privacy \cite{koufogiannis2015gradual}, defined as a Lipschitz bound on the log density of the mechanism's output.

\begin{definition}[Lipschitz Privacy]
Consider a normed space $(\mathcal{X},\|\cdot\|)$, privacy level $\epsilon > 0$, and response set $\mathcal{Z}$. A mechanism $\mathcal{M}\colon \mathcal{X} \to \mathcal{Z}$ is $\epsilon$-Lipschitz private if for all $Z \subseteq \mathcal{Z}$,
\[
\left|\ln \Pr[\mathcal{M}(x) \in Z] - \ln \Pr[\mathcal{M}(x') \in Z]\right| \leq \epsilon \|x - x'\|, \quad \forall x, x' \in \mathcal{X}.
\]
\end{definition}

\begin{definition}[Local Lipschitz Privacy]
Consider a normed space $(\mathcal{U}, \|\cdot\|)$, privacy level map $\epsilon\colon \mathcal{U} \to \mathbb{R}_+$, and response set $\mathcal{Y}$. A mechanism $Q\colon \mathcal{U} \to \Delta(\mathcal{Y})$ is $\epsilon(\cdot)$-Lipschitz private if for any $\mathcal{S} \subseteq \mathcal{Y}$ and $u \in \mathcal{U}$,
\[
\| \nabla_u \ln \Pr[Q(u) \in \mathcal{S}] \| \leq \epsilon(u).
\]
\end{definition}

\subsection{Equivalent Forms}
For mechanisms with differentiable probability density functions, Lipschitz privacy translates to pointwise gradient bounds. Let $g(\cdot; x)$ denote the probability density of $\mathcal{M}(x)$. The condition becomes:
\[
\|\nabla_x \ln g(z; x)\|_* \leq \epsilon \quad \forall x \in \mathcal{X}, z \in \mathcal{Z},
\]
where $\|\cdot\|_*$ is the dual norm. For $\ell_2$ norms (self-dual), this simplifies to:
\[
\left\|\nabla_{x_i} \ln g(z; x)\right\|_2 \leq \epsilon \quad \forall i \in \{1,\ldots,n\}.
\]

Consider private data $x = [x_1, \ldots, x_n]$ where each $x_i \in \mathbb{R}^m$. With the adjacency relation:
\[
(x, x') \in \mathcal{A} \iff \|x_i - x'_i\|_2 \leq \lambda \quad \forall i,
\]
we get the following equivalence:

\begin{proposition}
For any $\lambda > 0$, an $\epsilon$-Lipschitz private mechanism $\mathcal{M}$ is $(\epsilon\lambda)$-differentially private under adjacency relation $\mathcal{A}$.
\end{proposition}

\begin{proof}
See \cite{koufogiannis2015gradual} for proof details.
\end{proof}

Common differentially private mechanisms (e.g., Laplace, exponential) satisfy Lipschitz privacy. In our work, we use local Lipschitz privacy to prove $(\epsilon,\delta)$-differential privacy for our mechanism.

\begin{tcolorbox}[colback=cyan!10, colframe=black]
\begin{theorem}[Equivalence of privacy: Gradient \(\ell_2\) bound implies Lipschitz privacy] \label{equiTheorem}
Let \( g : \mathbb{R}^d \times \mathcal{Z} \to \mathbb{R}_{>0} \) be a conditional probability density function. Suppose that for each \( z \in \mathcal{Z} \), the function \( x \mapsto \ln g(x, z) \) is differentiable and satisfies
\[
\| \nabla_x \ln g(x, z) \|_2 \leq \varepsilon \quad \text{for all } x \in \mathbb{R}^d \text{ and all } z \in \mathcal{Z}.
\]
Then the mechanism \( x \mapsto g(x, \cdot) \) satisfies Lipschitz privacy with respect to the Euclidean norm, that is,
\[
|\ln g(x', z) - \ln g(x, z)| \leq \varepsilon \|x' - x\|_2 \quad \text{for all } x, x' \in \mathbb{R}^d \text{ and all } z \in \mathcal{Z}.
\]
\begin{proof}
   See Appendix \ref{equiProof}.
\end{proof}
\end{theorem}
\end{tcolorbox}

\section{Power Learning: Setup}
\label{sec:power_learn}Before formalizing our method, we detail the problem setup to provide the needed context. Consider a client with sensitive data $x$ who wishes to collaborate with a computationally powerful server for machine learning tasks, while maintaining privacy of their data. Our key insight is to transform this private data into embeddings $z$ through a carefully designed two-step process. First, a privatization network transforms $x$ into embeddings $z$ with quantifiable privacy guarantees derived from Lipschitz privacy. These embeddings are then evaluated by a utility network on the client side to ensure they retain task-relevant information. \par By jointly minimizing a privacy loss (which bounds the Lipschitz privacy of $z$) and a utility loss (which measures the task performance), we ensure the embeddings are both private and useful. The advantage of using Lipschitz privacy, is that we can account for the parameter $\epsilon$, post-hoc. This is because we can calculate the jacobians of the transformations which resulted in the embedding. The privacy loss, thus can be jointly optimized with utility, since the transformations to privatize the sample, can now be learnt using back-propogation. 

This approach creates a natural optimization framework: the embedding $z$ must balance between minimizing privacy loss to ensure stronger privacy guarantees, while preserving enough information to enable good performance on the utility network. Once these private embeddings are generated, they can be shared with the server in a single round of communication. The server, unconstrained by privacy requirements, can then employ any standard machine learning approach to process these embeddings for the desired task.

\subsection{Systems Interactions} 
\label{sysInt}
Power Learning operates through a carefully designed interaction between two entities: a client with sensitive data and a computationally powerful server. Figure \ref{mainFig} illustrates this interaction at a high level. The client employs a lightweight privatization network that transforms sensitive data into private embeddings. These embeddings come with formal privacy guarantees, established through a rigorous privacy calibration process. \par Specifically, we derive theoretical worst-case privacy bounds ($\epsilon$) from empirical measurements using high-probability confidence bounds. This calibration process transitions our estimated $\epsilon$   to $(\epsilon,\delta)$-differential privacy, providing a more practical privacy framework. Note that our method focuses on protecting the input data; we do not consider labels to be private in this work. The interaction process consists of two main components:
\subsection{Client-side Processing}
The client ensures privacy guarantees through two sequential stages.\\
 \textbf{Privacy-Inducing Training:} The client employs two networks: a privatization network that transforms input data $x$ into embeddings $z$, and a lightweight utility network that processes these embeddings for the learning task. The privatization network is trained to minimize two objectives: (1) a privacy loss that bounds the Lipschitz privacy of the generated embeddings $z$, derived from the transformation's gradient properties, and (2) a utility loss that measures how well these embeddings perform on the client's utility network. This joint optimization ensures that the embeddings $z$ maintain sufficient privacy (controlled by the privacy loss) while preserving enough information for the learning task (verified by the utility network). Both networks on the client side are intentionally lightweight, requiring significantly less computation than the server's model.
   \\ \textbf{Privacy Level Calibration:} Due to the non-convex nature of the joint loss function and its sample-size dependency, empirical loss minimization alone cannot guarantee differential privacy. We develop a theoretical framework to convert empirical privacy measurements into formal $(\epsilon,\delta)$-differential privacy guarantees (detailed in Appendices A and B). This calibration allows the client to verify the privacy level of each sample before transmission, ensuring only sufficiently private embeddings are shared with the server in a single communication round.

\begin{algorithm}
\caption{PowerLearn: Privacy-Preserving Collaborative Learning}
\begin{algorithmic}[1]
\renewcommand{\algorithmicindent}{1em}

\STATE {\color{BrickRed}\textbf{def}} {\color{MidnightBlue}\textbf{PowerLearn}}($X$, $Y$, $\epsilon$, $S$)
\STATE {\color{OliveGreen}\# Preprocess training data}
\STATE $X_{\text{train}}, Y_{\text{train}}, X_{\text{val}}, Y_{\text{val}} \leftarrow$ {\color{Plum}\textbf{PreprocessData}}($X$, $Y$)
\STATE {\color{OliveGreen}\# Initialize privacy model}
\STATE $\text{client\_model} \leftarrow$ {\color{Plum}\textbf{PrivacyNet}}($\text{depth}$)
\STATE $\text{trainer} \leftarrow$ {\color{Plum}\textbf{PrivacyTrainer}}($\text{client\_model}$, $\text{opt}$)
\STATE $\text{trainer}.${\color{Plum}\textbf{train}}()
\STATE {\color{OliveGreen}\# Generate embeddings}
\STATE $X_{\text{emb}}, X_{\text{val\_emb}} \leftarrow$ {\color{Plum}\textbf{GenerateEmbeddings}}($X_{\text{train}}, X_{\text{val}}$)
\STATE {\color{OliveGreen}\# Select corresponding labels}
\STATE $Y_{\text{emb}}, Y_{\text{emb\_val}} \leftarrow$ {\color{Plum}\textbf{SelectLabels}}($Y_{\text{train}}, Y_{\text{val}}$)
\STATE {\color{OliveGreen}\# Calibrate model}
\STATE {\color{Plum}\textbf{calibration}}($\text{client\_model}$, $X_{\text{train}}, X_{\text{val}}, \epsilon, S, Y$)
\STATE {\color{BrickRed}\textbf{return}} $X_{\text{emb}}, Y_{\text{emb}}$
\STATE
\STATE {\color{OliveGreen}\# Client Side}
\STATE $X_{\text{emb}}, Y_{\text{emb}} \leftarrow$ {\color{MidnightBlue}\textbf{PowerLearn}}($X$, $Y$, $\epsilon$, $S$)
\STATE {\color{OliveGreen}\# Send to server}
\STATE
\STATE {\color{OliveGreen}\# Server Side}
\STATE $\text{server\_model} \leftarrow$ {\color{Plum}\textbf{NeuralNet/XGBoost/RandForr}}()
\STATE $\text{trainer} \leftarrow$ {\color{Plum}\textbf{Trainer}}($\text{server\_model}$, $X_{\text{emb}}$, $Y_{\text{priv}}$, $\text{opt}$)
\STATE $\text{trainer}.${\color{Plum}\textbf{train}}()
\end{algorithmic}
\end{algorithm}
\subsubsection{Server-side Processing}
The server receives these private embeddings and enjoys complete flexibility in its choice of machine learning methods. As shown in Figure \ref{magtwo}, the server can employ any standard approach - neural networks, random forests, or XGBoost - making our framework model-agnostic. We demonstrate this flexibility through extensive empirical evaluation in Section [X], comparing performance across different server-side models.

\begin{figure}[!htbp]
    \centering
\includegraphics[width=0.8\linewidth]{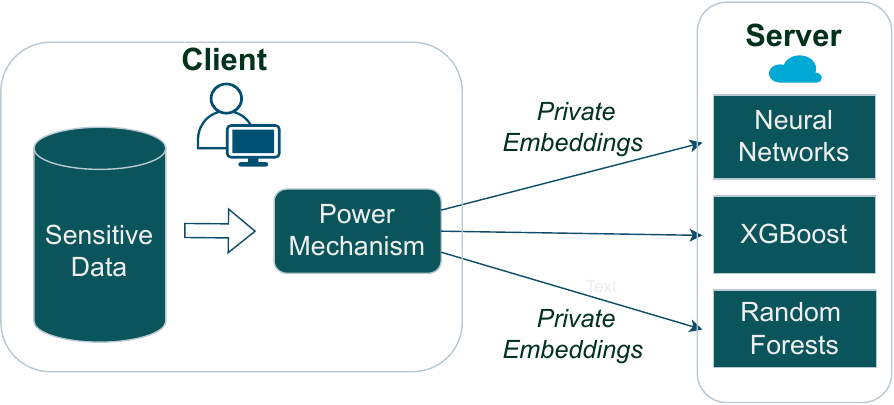}
    \caption{The interactions allow the server to use several machine learning methods, making the system private and fairly model agnostic.}
    \label{magtwo}
\end{figure}

\subsection{Power Mechanism: Generation of Private Embeddings} \label{sec:pow_mech_thm}One of the foundations of neural-networks has been to use several variants of compositions of functions to define them. Inspired by that, we provide a condition to be be enforced on compositions of functions over the raw data, to obtain $\epsilon$-Lipschitz privacy over the outputs with regards to the raw data. We then later on use this result to provide a method for releasing embeddings of tabular data with privacy.
This main result, proposed below, provides a sufficient condition on a composition of the form $\mathbf{z} = G(\mathbf{x})= g_{p-1}\circ g_{p-2}\ldots \circ  g_{0}(\mathbf{x})$, that ensures $\epsilon$-Lipschitz privacy on their outputs $\mathbf{z}$. 
\begin{tcolorbox}[colback=cyan!10, colframe=black]
\begin{theorem}(Power Learning Theorem)
Let $\mathbf{X} \in \mathbb{R}^{n \times d}$ be a data set where each sample $\mathbf{x} \in \mathbb{R}^d$ has a probability density function $f_{\mathcal{X}}(\mathbf{x})$.
Suppose $G=g_{p-1} \circ \cdots \circ g_0$ is a composition of $C^1$-diffeomorphisms $g_k: \mathbb{R}^d \rightarrow \mathbb{R}^d$ with Jacobians $\mathbf{J}_k\left(\mathbf{w}_{k-1}\right)=\frac{\partial g_k}{\partial \mathbf{w}_{k-1}}$ for $\mathbf{w}_{k-1}=g_{k-1} \circ \cdots \circ g_0(\mathbf{x})$. If the transformations satisfy

$$
\left\|\nabla_{\mathbf{x}} \log f_{\mathcal{X}}(\mathbf{x})-\sum_{k=0}^{p-1} \nabla_{\mathbf{x}} \log \left|\operatorname{det} \mathbf{J}_k\right|\right\| \leq \epsilon
$$then the output $\mathbf{z}=G(\mathbf{x})$ achieves $\epsilon$-Lipschitz privacy, defined as $
\left\|\nabla_{\mathbf{x}} \log h_{\mathcal{Z}}(\mathbf{z})\right\| \leq \epsilon$ where $h_{\mathcal{Z}}$ is the density of $\mathbf{z}$.
\end{theorem} 
\begin{proof}
    See Appendix \ref{powProof}.
\end{proof}
\end{tcolorbox}


\subsection{Sketch of proof strategy}
  We first provide a brief sketch of the proof strategy here before listing down the formal proof. The strategy is to first use kernel density estimates to model the input data distribution along with confidence bounds around it. Then the classical change of variable theorem for probability distributions is used to model the distribution of the output of the learned transformation, based on any given set of weights. Then these input and output probability distributions are used to put in a constraint on them to achieve $\epsilon$-Lipschitz privacy, by finding a loss function of the weights that needs to be minimized. This gives the result stated in the theorem. 

\subsection{Restricted functional form of embeddings} Although the above theorem is generic for different kinds of $g$' that are one-to-one and continuous, we now restrict ourselves to a specific functional form of $g$, which we use to apply power learning to neural networks
in the setting of collaborative learning as described in Section \ref{sysInt}. To cater to this setting, we use a multilayer perceptron to learn a matrix $\mathbf{H}$
where $\mathbf{H} = \mathbf{P_N(x)}$ for some input $\mathbf{x} \in \mathbb{R}^d$ where $\mathbf{P_N}$ is the neural network. Following our nomenclature from the above shared notation in Table \ref{notTab} we have,
$$g_k(\mathbf{w_k}) = \mathbf{H_kw_k} =\mathbf{P_N(w_k)w_k} $$
$$ \therefore \mathbf{z} = G(\mathbf{x}) = g_{p-1}\circ g_{p-2}\ldots \circ  g_{0}(\mathbf{x}) $$
Now we use $\mathbf{z}$ as an embedding and feed it as input to the smaller client utility network $\mathbf{U_N}$  to generate the predicted label $\mathbf{\tilde{y}}  =  \mathbf{U_N(z)} $. We want the embedding $\mathbf{z}$ to be generated in such a way that it has a formal guarantee of privacy, so that $\mathbf{P_N}$ rightly becomes a privatization network as described in Figure \ref{mainFig}.
\subsubsection{Privacy Inducing Loss funtion: Pre-Calibration}
We now need to jointly train the privacy network which generates the matrix and the utility network on the client. The joint loss can be divided into two parts.

$$\mathcal{L}_P(\mathbf{z},\mathbf{x}) = \left\lVert    \frac{\partial h_{\mathcal{Z}}(\mathbf{z})}{\partial \mathbf{x}}  \right \rVert =  \left\lVert   \frac{f'_{\mathcal{X}}(\mathbf{x})}{f_{\mathcal{X}}(\mathbf{x})} - \sum_{k=0}^{p-1} \frac{\partial}{\partial \mathbf{x}} \log(|\det(\mathbf{J}_k)) \right \rVert $$
The utility loss function $\mathcal{L}_U(y,\tilde{y})$ depends on the task. Combining the two losses gives us our joint loss function.
\begin{equation}
\mathcal{L}(\mathbf{y},\mathbf{x})  =  \left\lVert   \frac{f'_{\mathcal{X}}(\mathbf{x})}{f_{\mathcal{X}}(\mathbf{x})} - \sum_{k=0}^{p-1} \frac{\partial}{\partial \mathbf{x}} \log(|\det(\mathbf{J}_k)) \right \rVert + \lambda  \mathcal{L}_U(y,\tilde{y})    
\end{equation}

Upon minimization, the conversion from the empirically measured privacy level to an exact theoretically guaranteed privacy level of $\epsilon$ is performed as detailed in Appendices 1.) and 2.) of Appendix A.

\subsection{Calibration of attained $\epsilon$ level of privacy}

We use kernel density estimation to estimate the probability density of each sample as given by
$ \hat{f}_{\mathcal{X}}(x) = \frac{1}{nh^d} \sum^{n}_{i=1} K\left(\frac{x- X_i}{h}\right)$. The Gaussian kernel here is given by
$K(u) = \frac{e^{- \lvert \lvert u \rvert \rvert^2}}{(2 \pi)^{d/2} }$. This helps us to account for the term $\frac{f'_{\mathcal{X}}(\mathbf{x})}{f_{\mathcal{X}}(\mathbf{x})}$ in the loss of privacy. However, we need to find confidence intervals for these probability density estimates to understand the worst case $\epsilon$. The range in which the true probability density lies with $1-\alpha$ probability is given by
$$ CI_{1-\alpha} = [\hat{f}_{\mathcal{X}}(\mathbf{x}) - z_{1-\alpha/2} \sqrt{\frac{\mu_{K} \hat{f}_{\mathcal{X}}(\mathbf{x})}{nh^d}},\hat{f}_{\mathcal{X}}(\mathbf{x})+z_{1-\alpha/2} \sqrt{\frac{\mu_{K} \hat{f}_{\mathcal{X}}(\mathbf{x})}{nh^d}}].$$
The term $\mu_K$ is given by
$\mu_K =  \int K^{2}(x) dx$.
For the Gaussian kernel, this is evaluated as $\mu_K =  1/(2^{d} \pi^{d/2})$. The condition for $\epsilon $ Lipschitz privacy is given by
$  \left\lVert \frac{\partial}{\partial \mathbf{x}}\log h_{\mathcal{Z}}(\mathbf{z})\right\rVert \leq \epsilon$. Hence,
to obtain Lipschitz privacy on estimated probability with confidence $ 1- \alpha$ we have the condition to be,

$$
\begin{alignedat}{2}
&\left\lVert
\frac{\partial f_{\mathcal{X}}(\mathbf{x})}{f_{\mathcal{X}}(\mathbf{x}) \partial \mathbf{x}} -  \frac{\partial}{\partial \mathbf{x}} \sum_{k=0}^{p-1}  \log(|\det(\mathbf{J}k))\right \rVert = 
&\left\lVert
\frac{\partial \hat{f}{\mathcal{X}}(\mathbf{x})}{f_{\mathcal{X}}(\mathbf{x}) \partial \mathbf{x}} + \frac{\partial f_{\mathcal{X}}(\mathbf{x}) - \partial \hat{f}{\mathcal{X}}(\mathbf{x})}{f{\mathcal{X}}(\mathbf{x}) \partial \mathbf{x}}-  \frac{\partial}{\partial \mathbf{x}} \sum_{k=0}^{p-1}  \log(|\det(\mathbf{J}_k))\right \rVert \leq \epsilon
\end{alignedat}.
$$

This simplifies as follows based on the Cauchy-Schwartz inequality,


$$
\begin{alignedat}{2}
&\left\lVert \frac{\partial}{\partial \mathbf{x}}\log h_{\mathcal{Z}}(\mathbf{z}) \right \rVert  = \left\lVert
\frac{\partial \hat{f}{\mathcal{X}}(\mathbf{x})}{f{\mathcal{X}}(\mathbf{x}) \partial \mathbf{x}} -  \frac{\partial}{\partial \mathbf{x}} \sum_{k=0}^{p-1}  \log(|\det(\mathbf{J}k))\right \rVert + 
&\left\lVert \frac{\partial f{\mathcal{X}}(\mathbf{x}) - \partial \hat{f}{\mathcal{X}}(\mathbf{x})}{f{\mathcal{X}}(\mathbf{x}) \partial \mathbf{x}} \right \rVert \leq \epsilon
\end{alignedat}.
$$
Now upon using the above stated confidence interval bounds on $f(X)$, we can estimate the effectively obtained privacy level as $\epsilon^{\prime} +  \left\lVert \frac{\partial f_{\mathcal{X}}(\mathbf{x}) - \partial \hat{f}_{\mathcal{X}}(\mathbf{x})}{f_{\mathcal{X}}(\mathbf{x}) \partial \mathbf{x}} \right \rVert $ with $\epsilon^{\prime}$ in the form of 
$\epsilon^{\prime}  = \max\Bigg(lower, upper \Bigg) $ where, $$lower=\left\lVert 
 \frac{\partial \hat{f}_{\mathcal{X}}(\mathbf{x})}{\Big(\hat{f}_{\mathcal{X}}(\mathbf{x}) -z_{1-\alpha/2} \sqrt{\frac{\mu_{K} \hat{f}_{\mathcal{X}}(\mathbf{x})}{nh^d}}\Big) \partial \mathbf{x}} -  \frac{\partial}{\partial \mathbf{x}} \sum_{k=0}^{p-1}  \log(|\det(\mathbf{J}_k))\right \rVert $$
and, $$upper = \left\lVert 
 \frac{\partial \hat{f}_{\mathcal{X}}(\mathbf{x})}{\Big(\hat{f}_{\mathcal{X}}(\mathbf{x}) + z_{1-\alpha/2} \sqrt{\frac{\mu_{K} \hat{f}_{\mathcal{X}}(\mathbf{x})}{nh^d}}\Big) \partial \mathbf{x}} -  \frac{\partial}{\partial \mathbf{x}} \sum_{k=0}^{p-1}  \log(|\det(\mathbf{J}_k))\right \rVert.$$

Now for $K = \mu_K/nh^d $ since $\left\lVert \frac{\partial f_{\mathcal{X}}(\mathbf{x}) - \partial \hat{f}_{\mathcal{X}}(\mathbf{x})}{f_{\mathcal{X}}(\mathbf{x}) \partial \mathbf{x}} \right \rVert \leq d \left\vert  \sqrt{\frac{K}{4\hat{f}_{\mathcal{X}}(x)}}  \right \rVert \mathcal \mathcal{z}_{1 - \alpha /2}  \text{ with probability $1- \alpha$}$, we have the final effective privacy level $\epsilon$ to be given by the following upper bound,
$$ \epsilon \leq  \epsilon' + d \left\vert  \sqrt{\frac{K}{4\hat{f}_{\mathcal{X}}(x)}}  \right \rVert \mathcal \mathcal{z}_{1 - \alpha /2}  \text{ with probability $1- \alpha$}.  $$


\section*{Reconstruction prevention under Lipschitz privacy}

\begin{tcolorbox}[colback=cyan!10, colframe=black]
\begin{lemma}
Let $A(z) \in \mathbb{R}^d$ be a vector-valued random variable and let $\mu(x) = \mathbb{E}_{z \sim p_Z(\cdot \mid x)}[A(z)]$ denote its conditional mean given $x$. Then,
\[
\mathbb{E}_{z \sim p_Z(\cdot \mid x)}\left[\|A(z) - \mu(x)\|^2\right] = \operatorname{Tr}(\operatorname{Cov}(A(z) \mid x))
\]
\end{lemma}
\begin{proof}
    See Appendix xyz
\end{proof}
\end{tcolorbox}

\begin{tcolorbox}[colback=cyan!10, colframe=black]
\begin{lemma}
Let $f_X(x)$ be a continuously differentiable probability density function on $\mathbb{R}^d$ that decays sufficiently rapidly at infinity, such that $f_X(x) \to 0$ and $\nabla_x f_X(x) \to 0$ as $\|x\| \to \infty$. Then,
\[
\mathbb{E}_{x \sim f_X}[\nabla_x \log f_X(x)] = 0.
\]
\end{lemma}
\begin{proof}
    See Appendix xyz
\end{proof}
\end{tcolorbox}

Let \( X \in \mathbb{R}^d \) be a random variable with density \( f_X(x) \), and let \( Z = G(X) \in \mathbb{R}^m \) be the output of a randomized mechanism with conditional density \( p_Z(z \mid x) \). Suppose,
 \( \log f_X(x) \) and \( \log p_Z(z \mid x) \) are twice differentiable,
   the mechanism satisfies the pointwise gradient bound \( \|\nabla_x \log p_Z(z \mid x)\|_2^2 \leq \varepsilon^2 \),
 the Fisher information matrix \( \mathcal{I}(f_X) = \mathbb{E}_{x}[\nabla_x \log f_X(x) \nabla_x \log f_X(x)^T] \) is symmetric and finite. Let \( A : \mathbb{R}^m \to \mathbb{R}^d \) be any estimator, with conditional mean \( \mu(x) = \mathbb{E}[A(z) \mid x] \) and Jacobian \( J_\mu(x) = \nabla_x \mu(x) \). Assume, further that the null spaces of \( \mathcal{I}_{Z|X}(x) \) and \( \mathcal{I}(f_X) \) intersect trivially. Then the reconstruction error satisfies, the lower bound stated below. 
\begin{tcolorbox}[colback=cyan!10, colframe=black]

\begin{theorem}
The reconstruction error is lower-bounded as follows.
\[
R(A) = \mathbb{E}_{x,z}[\|A(z) - x\|^2] \geq \mathbb{E}_x[\operatorname{Tr}(J_\mu(x)(\mathcal{I}_{Z|X}(x) + \mathcal{I}(f_X))^{-1} J_\mu(x)^T) + \|\mu(x) - x\|^2].
\]
In the special case where \( \mu(x) = x \), this simplifies to
\[
R(A) \geq \operatorname{Tr}((\mathcal{I}_{Z|X}(x) + \mathcal{I}(f_X))^{-1}) \geq \frac{d^2}{\varepsilon^2 + \operatorname{Tr}(\mathcal{I}(f_X))}.
\]
\end{theorem}
\begin{proof}
    See Appendix xyz
\end{proof}
\end{tcolorbox}

\section{Empirical Calibration of the reconstruction prevention bound}
Let $X \in \mathbb{R}^d$ be a random variable with unknown density $f_X(x)$, and let $x_1, \dots, x_n$ be i.i.d. samples drawn from $f_X$. Let $K : \mathbb{R}^d \to \mathbb{R}_{\geq 0}$ be a continuously differentiable, symmetric kernel with compact support and finite second moments. For a bandwidth $h > 0$, define the kernel density estimator
$
\widehat{f}_X(x) := \frac{1}{n h^d} \sum_{i=1}^n K\left( \frac{x - x_i}{h} \right).
$
Let $\widehat{s}(x) := \nabla_x \log \widehat{f}_X(x)$ be the estimated score function and define the empirical Fisher information estimator by
$
\widehat{\mathcal{I}}(f_X) := \frac{1}{n} \sum_{i=1}^n \widehat{s}(x_i) \widehat{s}(x_i)^\top.
$
Then under standard conditions on $K$, $f_X$, and the bandwidth $h$ (e.g., $h \to 0$, $nh^d \to \infty$), the estimator $\widehat{\mathcal{I}}(f_X)$ converges in probability to the true Fisher information matrix $\mathcal{I}(f_X)$. 
\begin{tcolorbox}[colback=cyan!10, colframe=black]
\begin{theorem}
The reconstruction error of any such density estimator $A$ under an $\varepsilon$-Lipschitz mechanism satisfies
\[
\mathcal{R}(A) \geq \frac{d^2}{\varepsilon^2 + \operatorname{Tr}(\widehat{\mathcal{I}}(f_X)) + \frac{c_1^2}{n h^{d+4}}},
\]
where $c_1 > 0$ is a constant depending on the bias $b_d(x)$ and variance $\sigma_d^2(x)$ of the kernel estimator and on kernel shape.
\end{theorem}
\begin{proof}
    See Appendix xyz
\end{proof}
\end{tcolorbox}

\subsection{Additional Experiments and Results} 
\label{addexpres}
\begin{figure*}[!htbp]
    \centering
\begin{floatrow}\ffigbox{\includegraphics[scale=0.41]{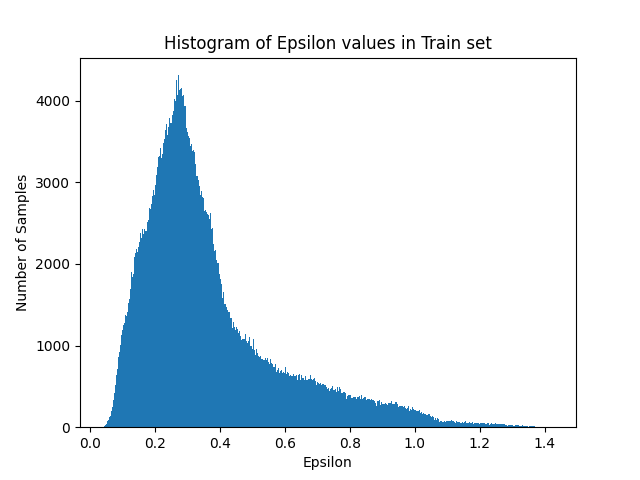}}{\caption{Train Histogram of $\epsilon$ for PowerLearn Embeddings}\label{fig:train_hist}}
      \caption{Train Histogram of $\epsilon$ for PowerLearn Embeddings}
    \label{fig:train_hist}
    \includegraphics[scale=0.41]{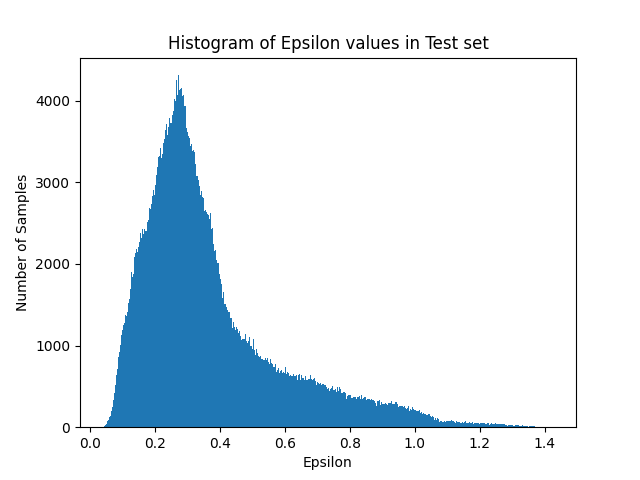}
    \label{fig:test_hist}
    \end{floatrow}

\end{figure*}

\begin{figure*}[!htbp]
    \centering
    \begin{floatrow}
    \includegraphics[scale=0.04]{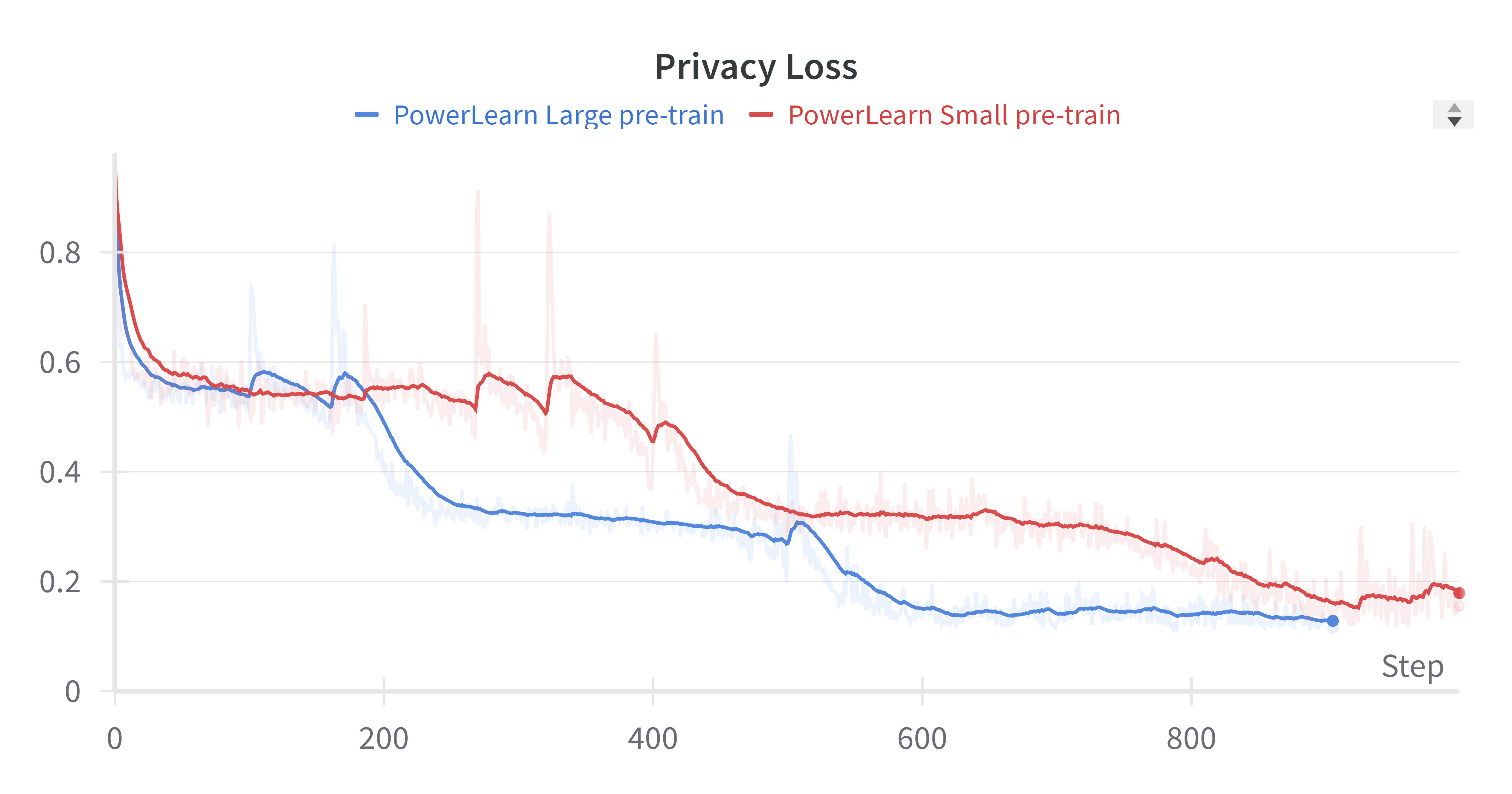}
      \caption{Convergence of privacy and utility losses on client model}
    \label{fig:priv_conv}
     \includegraphics[scale=0.04]{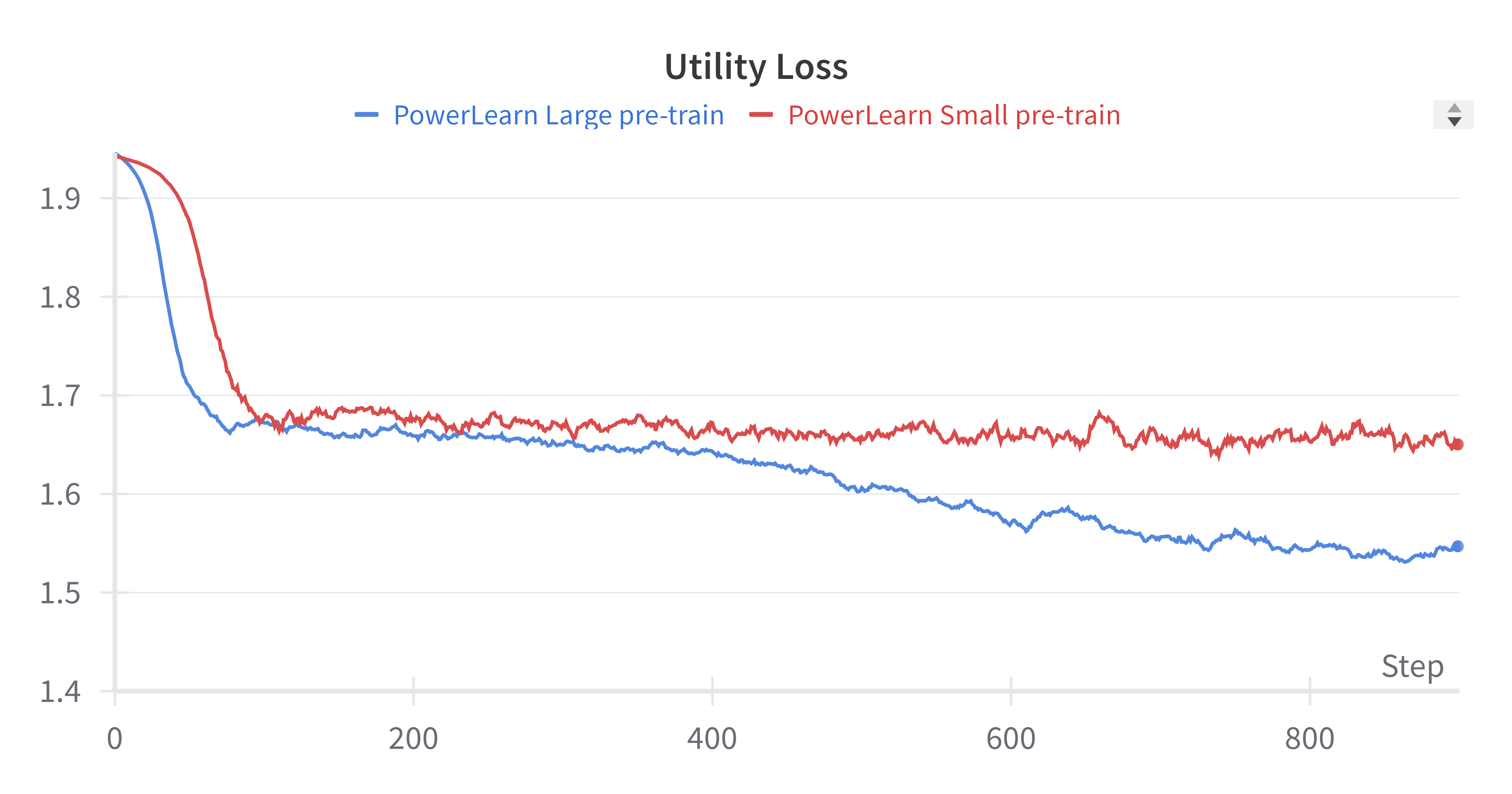}
    \label{fig:priv_loss_conv}
    \end{floatrow}
 \end{figure*}
 \begin{figure*}[!htbp]
    \centering
    \begin{floatrow}
    \includegraphics[scale=0.04]{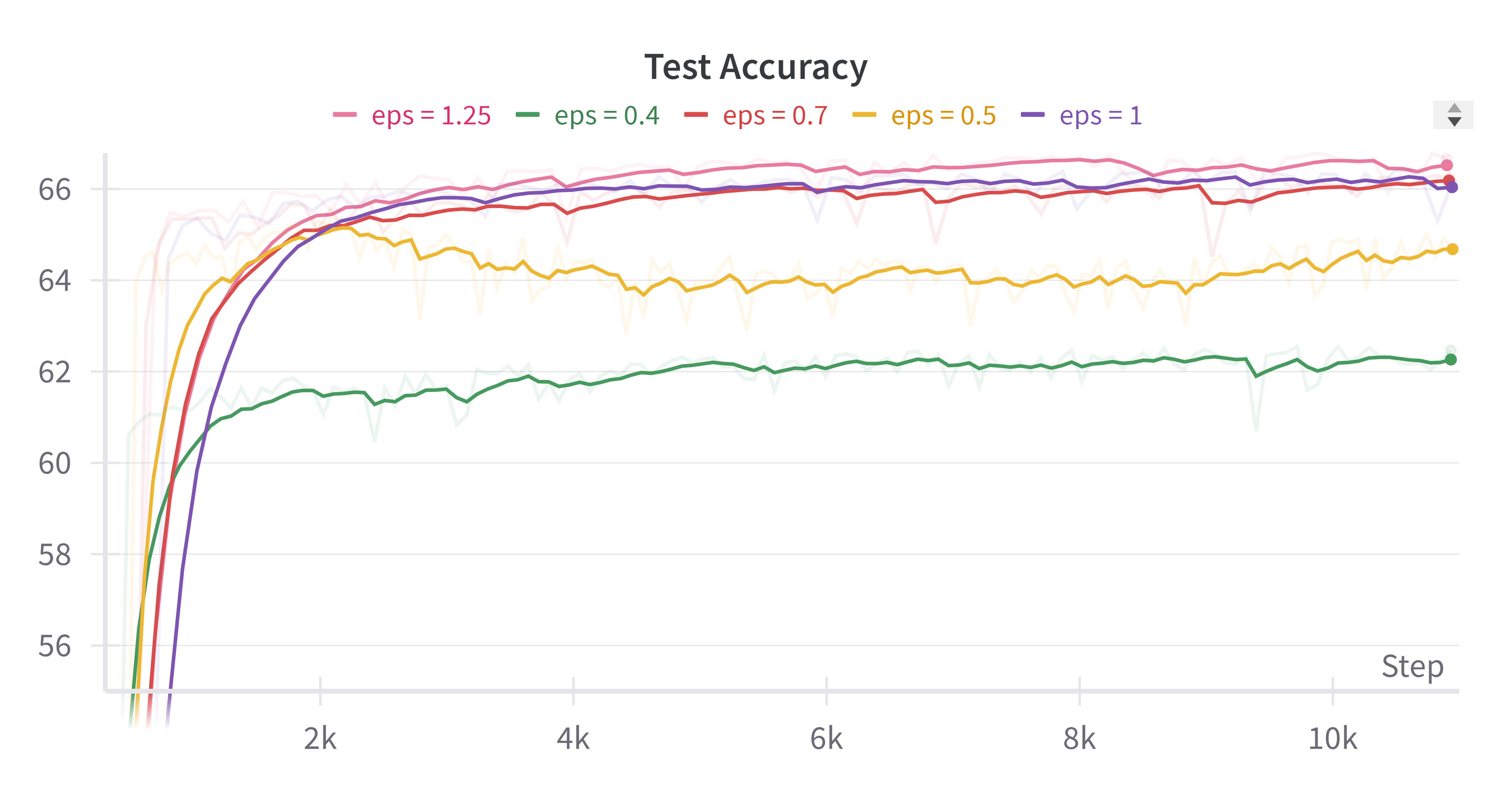}
      \caption{Convergence of test accuracies on the server model}
    \label{fig:priv_conv}
     \includegraphics[scale=0.04]{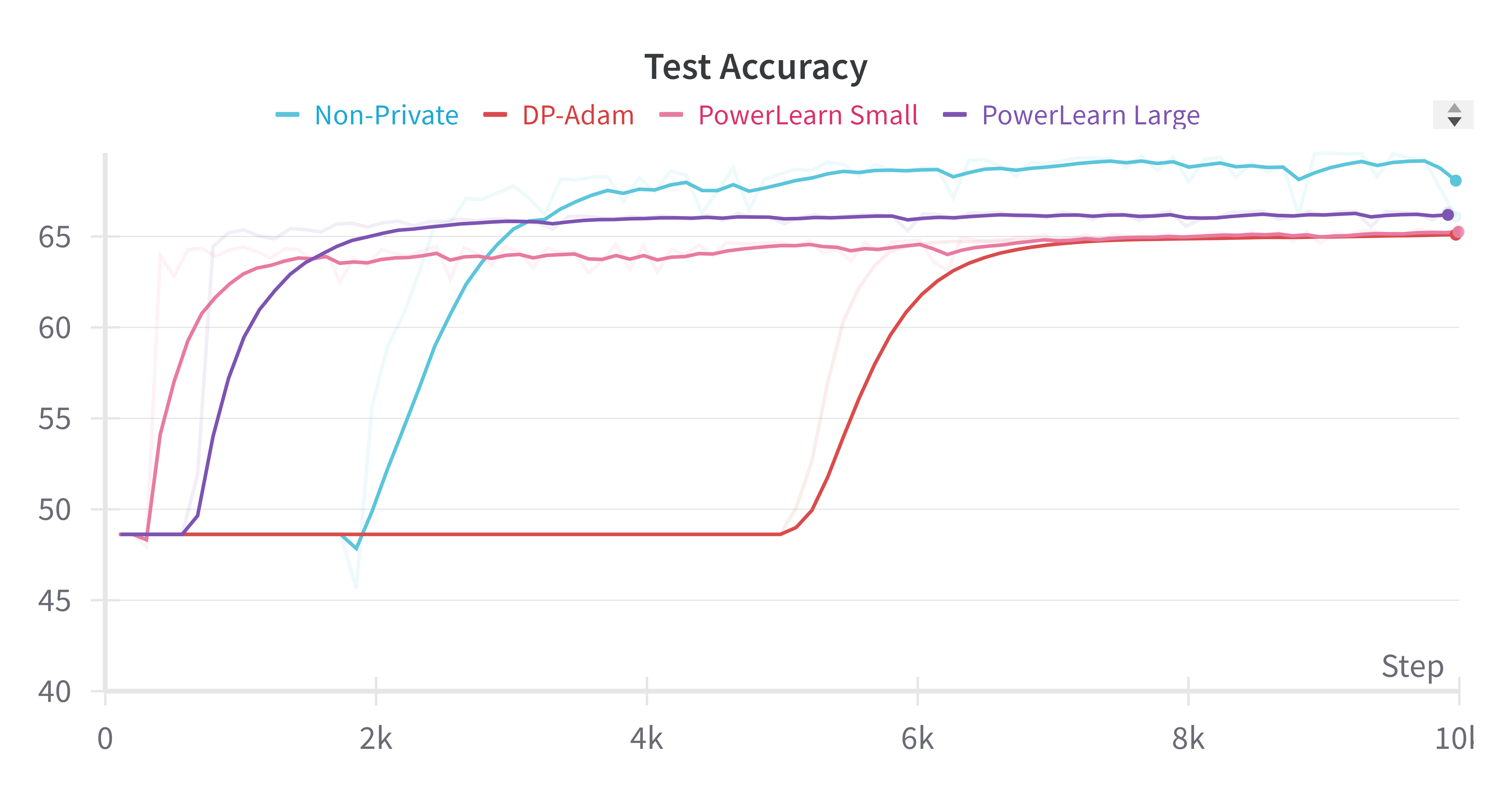}
    \label{fig:priv_acc_conv}
    \end{floatrow}
 \end{figure*}
 \begin{figure}[ht]
    \centering
    {\includegraphics[scale=0.05]{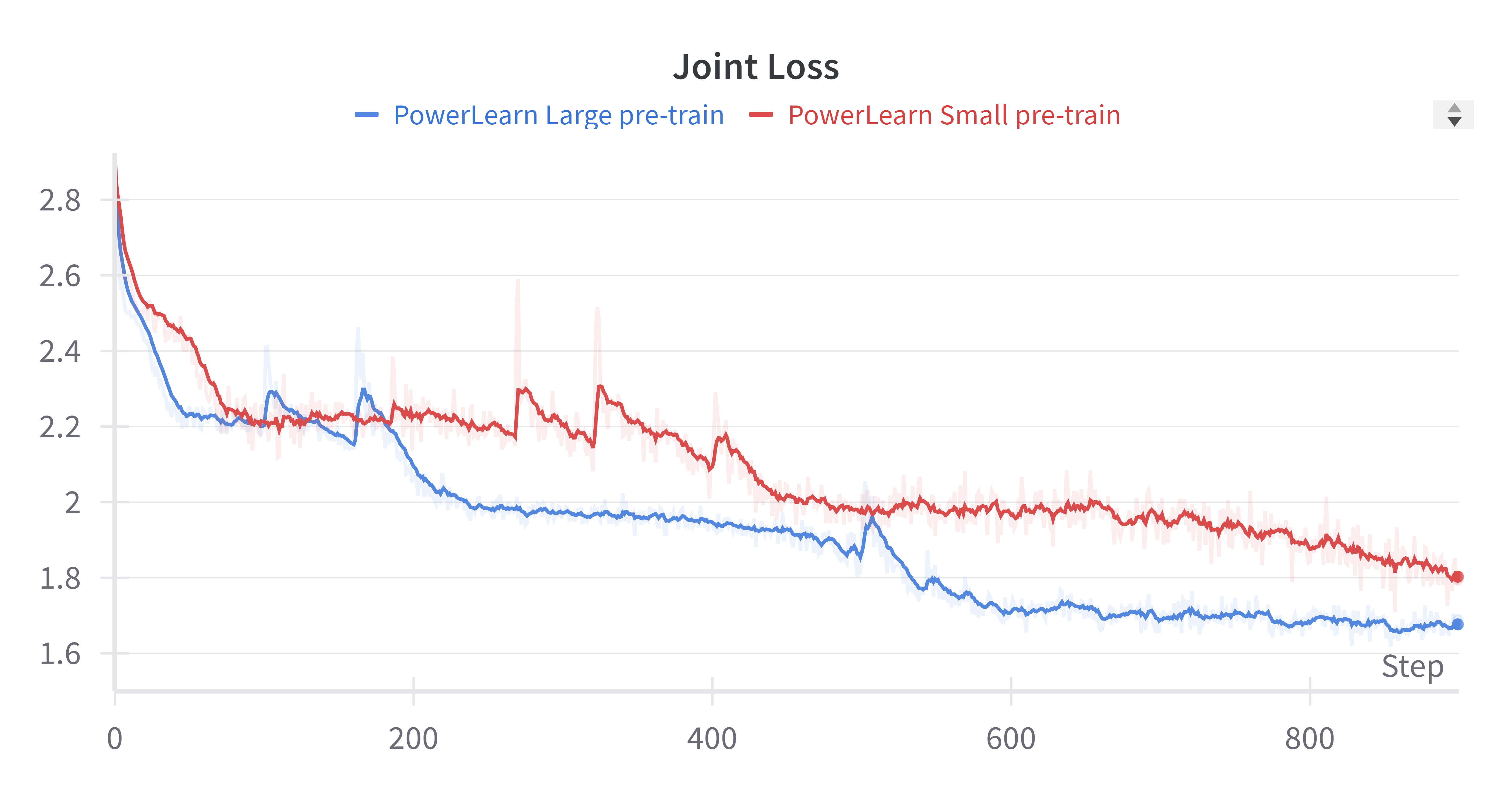}}\caption{Convergence of the joint loss used based on a combination of the privacy loss and utility loss.}
    \label{fig:foobar}
\end{figure}

\textbf{Theorem (Convergence of the Power Mechanism for a Two-Layer Neural Network with One Linear and One Nonlinear Layer).}

We now specialize to the case where the privatizer is a two-layer multilayer perceptron given by
\[
G_\theta(x) = \phi(W_2 W_1 x + b_2)
\]
where \( \phi = \tanh \) is applied elementwise. The first layer is given by \( g_0(x) = W_1 x \), which has Jacobian \( J_0 = W_1 \), and the second layer is given by \( g_1(w) = \phi(W_2 w + b_2) \), with Jacobian \( J_1 = D_2 W_2 \), where \( D_2 = \mathrm{diag}(\phi'(a)) \) and \( a = W_2 W_1 x + b_2 \). Since \( W_1 \) is constant with respect to \( x \), we have \( \nabla_x \log |\det J_0| = 0 \), so the only contribution to the privacy loss comes from \( J_1 \).

We now compute
\[
\log |\det D_2| = \sum_{i=1}^m \log \phi'(a_i)
\]
which implies
\[
\nabla_x \log |\det D_2| = \sum_{i=1}^m \frac{\phi''(a_i)}{\phi'(a_i)} \nabla_x a_i.
\]
Since each \( a_i = W_{2, i, :} W_1 x + b_{2,i} \), we have \( \nabla_x a_i = W_{2,i,:} W_1 \). Thus,
\[
\nabla_x \log |\det J_1(x)| = \sum_{i=1}^m \frac{\phi''(a_i)}{\phi'(a_i)} W_{2,i,:} W_1.
\]
This gives the exact expression for the privacy loss:
\[
\mathcal{L}_P(\theta) = \left\| \nabla_x \log f_X(x) - \sum_{i=1}^m \frac{\phi''(a_i)}{\phi'(a_i)} W_{2,i,:} W_1 \right\|_2.
\]

To bound this expression, we observe that for \( \phi(z) = \tanh(z) \), we have
\[
\phi'(z) = 1 - \tanh^2(z)
\quad \text{and} \quad
\phi''(z) = -2 \tanh(z)(1 - \tanh^2(z)).
\]
Therefore,
\[
\left| \frac{\phi''(z)}{\phi'(z)} \right| = 2 |\tanh(z)| \le 2
\]
because \( \tanh(z) \in (-1, 1) \). Letting \( \xi_i = \frac{\phi''(a_i)}{\phi'(a_i)} \), we obtain
\[
\left\| \sum_{i=1}^m \xi_i W_{2,i,:} W_1 \right\|
\le \sum_{i=1}^m |\xi_i| \cdot \|W_{2,i,:} W_1\|
\le 2 \sum_{i=1}^m \|W_{2,i,:} W_1\|.
\]
By submultiplicativity of matrix norms, we have
\[
\|W_{2,i,:} W_1\| \le \|W_{2,i,:}\|_2 \cdot \|W_1\|_2.
\]
Hence,
\[
\sum_{i=1}^m \|W_{2,i,:} W_1\| \le \|W_1\|_2 \sum_{i=1}^m \|W_{2,i,:}\|_2.
\]
Applying the Cauchy-Schwarz inequality gives
\[
\sum_{i=1}^m \|W_{2,i,:}\|_2 \le \sqrt{m} \left( \sum_{i=1}^m \|W_{2,i,:}\|_2^2 \right)^{1/2} = \sqrt{m} \cdot \|W_2\|_F.
\]
If we assume \( \|W_1\|_2 \le \sqrt{h} \) and \( \|W_2\|_F \le \sqrt{m} \), then we obtain
\[
\left\| \sum_{i=1}^m \frac{\phi''(a_i)}{\phi'(a_i)} W_{2,i,:} W_1 \right\| \le 2 \sqrt{m} \cdot \sqrt{m} \cdot \sqrt{h} = 2 m \sqrt{h}.
\]
Therefore, we conclude that
\[
L_P \le (2 m \sqrt{h})^2 = 4 m^2 h.
\]

We now turn to bounding the smoothness of the utility loss. Let
\[
\mathcal{L}_U(y, \hat{y}) = -\sum_i y_i \log \hat{y}_i
\]
with prediction
\[
\hat{y} = \mathrm{softmax}(W_3 z + b_3).
\]
Then
\[
\nabla_z \mathcal{L}_U = W_3^\top (\hat{y} - y)
\]
and so
\[
\|\nabla_z \mathcal{L}_U\| \le \|W_3\|_2 \le \sqrt{c}.
\]
Since \( z = \phi(W_2 W_1 x + b_2) \), the chain rule yields
\[
\left\| \frac{\partial z}{\partial \theta} \right\| \le \|W_1\|_2 \cdot \|D_2\| \le \sqrt{h}
\]
because \( \|D_2\| \le 1 \). Hence
\[
L_U \le \|W_3\|_2^2 \cdot \left\| \frac{\partial z}{\partial \theta} \right\|^2 \le c h.
\]

Finally, we analyze convergence of stochastic gradient descent. Let \( \mathcal{L}(\theta) = \mathcal{L}_P(\theta) + \lambda \mathcal{L}_U(\theta) \), and suppose that
\[
\mathbb{E}[g_t \mid \theta_t] = \nabla \mathcal{L}(\theta_t)
\quad \text{and} \quad
\mathbb{E}[\|g_t - \nabla \mathcal{L}(\theta_t)\|^2] \le \sigma^2.
\]
Let the step size \( \eta \) satisfy \( \eta < 1/L \), where \( L = 4 m^2 h + \lambda c h \). The descent lemma for \(L\)-smooth functions gives
\[
\mathbb{E}[\mathcal{L}(\theta_{t+1})] \le \mathbb{E}[\mathcal{L}(\theta_t)] - \left(\eta - \frac{L \eta^2}{2}\right) \mathbb{E}[\|\nabla \mathcal{L}(\theta_t)\|^2] + \frac{L \eta^2 \sigma^2}{2}.
\]
Summing over \( t = 0 \) to \( T - 1 \) and dividing by \( T \), we find
\[
\frac{1}{T} \sum_{t=0}^{T-1} \mathbb{E}[\|\nabla \mathcal{L}(\theta_t)\|^2]
\le \frac{\mathcal{L}(\theta_0) - \mathcal{L}^*}{\eta(1 - L \eta / 2) T} + \frac{L \eta \sigma^2}{2(1 - L \eta / 2)}.
\]
Therefore, the convergence rate of SGD applied to the Power Mechanism loss is
\[
\min_{0 \le t < T} \mathbb{E}[\|\nabla \mathcal{L}(\theta_t)\|^2]
\le \frac{\mathcal{L}(\theta_0) - \mathcal{L}^*}{\eta(1 - L \eta / 2) T} + \frac{L \eta \sigma^2}{2(1 - L \eta / 2)}.
\]

\section{Experiments}   \label{sec:exp}
\begin{table*}[!htbp]
    \centering
    \begin{tabular}{|c|c|c|c|c|} \hline 
         $\epsilon$&  \textbf{PL-NN}&  \textbf{PL-RF}&  \textbf{PL-XGB}& \textbf{DP-ADAM}
\\ \hline 
0.35& 52.98 $\pm$ 0.02 & 65.96 $\pm$ 0.49 & 71.72 $\pm$ 0.20 & 64.81 $\pm$ 0.01
\\\hline 
0.40& 63.26 $\pm$ 0.91 & 66.95 $\pm$ 0.15 & 76.42 $\pm$ 0.03 & 64.89 $\pm$ 0.06
\\ \hline 
0.50& 65.07 $\pm$ 0.47 & 69.58 $\pm$ 0.49 & 81.94 $\pm$ 0.24 & 65.10 $\pm$ 0.14
\\ \hline 
0.70& 66.25 $\pm$ 0.02 & 73.42 $\pm$ 0.15 & 83.98 $\pm$ 0.29 & 65.38 $\pm$ 0.10
\\ \hline 
1.00& 66.79 $\pm$ 0.21 & 73.81 $\pm$ 0.42 & 85.71 $\pm$ 0.21 & 65.43 $\pm$ 0.13
\\ \hline
1.25& 67.00 $\pm$ 0.06 & 74.02 $\pm$ 0.28 & 85.84 $\pm$ 0.13 & 65.62 $\pm$ 0.04
\\ \hline
    \end{tabular}
    \label{tab:priv_acc_cover}

    \centering
    \begin{tabular}{|c|c|c|c|c|} \hline 
        $\epsilon$&  \textbf{PL-NN}&  \textbf{PL-RF}&  \textbf{PL-XGB}& \textbf{DP-ADAM}
\\ \hline 
0.50& 70.73 $\pm$ 0.02 & 62.51 $\pm$ 2.23 & 56.58 $\pm$ 6.89 & 69.58 $\pm$ 0.49
\\ \hline 
0.75& 74.22 $\pm$ 0.10 & 73.25 $\pm$ 1.90 & 73.73 $\pm$ 1.89 & 70.06 $\pm$ 0.30
\\ \hline 
1.20& 81.77 $\pm$ 0.45 & 79.55 $\pm$ 0.78 & 79.32 $\pm$ 0.83 & 71.27 $\pm$ 1.08
\\ \hline 
1.50& 82.30 $\pm$ 0.09 & 81.46 $\pm$ 0.22 & 81.75 $\pm$ 0.29 & 72.62 $\pm$ 2.07
\\ \hline
    \end{tabular}
    \label{tab:priv_acc_higg}
    \centering
    \begin{tabular}{|c|c|c|c|c|} \hline 
         $\epsilon$&  \textbf{PL-NN}&  \textbf{PL-RF}&  \textbf{PL-XGB}& \textbf{DP-ADAM}
\\ \hline 
0.70& 78.95 $\pm$ 0.13 & 80.41 $\pm$ 0.26 & 77.00 $\pm$ 2.94 & 76.06 $\pm$ 0.01
\\ \hline 
1.00& 81.77 $\pm$ 0.13 & 80.94 $\pm$ 0.14 & 79.41 $\pm$ 1.21 & 76.09 $\pm$ 0.04
\\ \hline 
1.50& 82.14 $\pm$ 0.25 & 81.81 $\pm$ 0.14 & 81.19 $\pm$ 0.96 & 78.41 $\pm$ 1.43
\\ \hline 
3.00& 82.84 $\pm$ 0.05 & 82.27 $\pm$ 0.10 & 82.00 $\pm$ 0.55 & 82.40 $\pm$ 0.15
\\ \hline
    \end{tabular}
    \caption{Utility vs Epsilon: Forest Cover (top table), Higgs Boson (middle table) and Adult Income Datasets (bottom table).}
    \label{tab:priv_acc_inc}
\end{table*}

In our experiments, we evaluate our method against established differentially private training approaches, including those facilitating model weight release rather than activation release as in our case.
\label{subsec:expt_set} To assess the efficacy of our approach for private embedding sharing in collaborative learning, we benchmark it against conventional private and non-private training methods, including simple split learning-based techniques known to lack privacy safeguards. Our results demonstrate that our method effectively balances computational load between server and client while preserving client data privacy, thereby optimizing utility. For comprehensive information on datasets, experimental parameters, and supplementary findings, refer to Appendix \ref{experimentsDetails}.
\\\textbf{Datasets.}  \label{subsec:expt_data}
Our experiments are implemented on three publicly available tabular datasets: Forest Cover Type, Higgs Boson, and Census Income. The Forest Cover Type dataset challenges models to predict forest cover categories using environmental variables such as soil composition and elevation. In the Higgs Boson dataset, the task involves distinguishing signal events indicative of Higgs boson production from background noise. The Census Income dataset requires predicting whether an individual's income surpasses the \$50,000 threshold based on demographic attributes. This diverse selection of datasets and tasks serves to illustrate the versatility and effectiveness of our proposed method. For a more comprehensive overview of these datasets, refer to the Appendix.\ref{appendix:dataset}.
\\\textbf{Baselines.} We use DP-ADAM model for private neural networks and train it entirely on the client side, to compare our embedding-release approach against weight-release paradigm.  For the non-private baseline, we use the same models on both client and server that are used for our method, except using the loss function and creating private embeddings. We call this model NonPriv.
\\\textbf{Models.} PowerLearn (PL) refers to model trained using {\ours} method. For the Forest Cover Dataset, we train two models PowerLearn small and PowerLearn large, which differ in batch size and number of steps on the client side. PowerLearn-NeuralNetwork (PL-NN), PowerLearn-RandomForest (PL-RF) and PowerLearn-XGBoost (PL-XGB) are models trained on privated embeddings generated using {\ours} and having a Neural Network,Random Forest classifier and an XGBoost classifier on the server side respectively.

\subsection{Experiment 1: Privacy vs. Utility trade-offs} \label{subsec:priv_acc}
In order to check the utility of the resulting server's model that is obtained while preserving client's data privacy, we measure the accuracy of the network as we vary the privacy parameter $\epsilon$. Three server models are evaluated: PowerLearn-NeuralNetwork (PL-NN), PowerLearn-RandomForest (PL-RF) and PowerLearn-XGBoost (PL-XGB). We compare our approach to the weight-release baseline DP-ADAM. The results are summarized in  Table \ref{tab:priv_acc_inc}. We note that {\ours} ensures much better privacy-utility tradeoff as compared to DP-Adam. The variation is strongly correlated to the $\epsilon$ histogram as it dictates the number of training points, the server receives, thus driving the accuracies. Figure \ref{fig:eps_vs_acc} shows the variation of accuracy for the 
Forest Cover dataset. 

\begin{figure}[h!]
    \centering
    \includegraphics[width=1.0\linewidth]{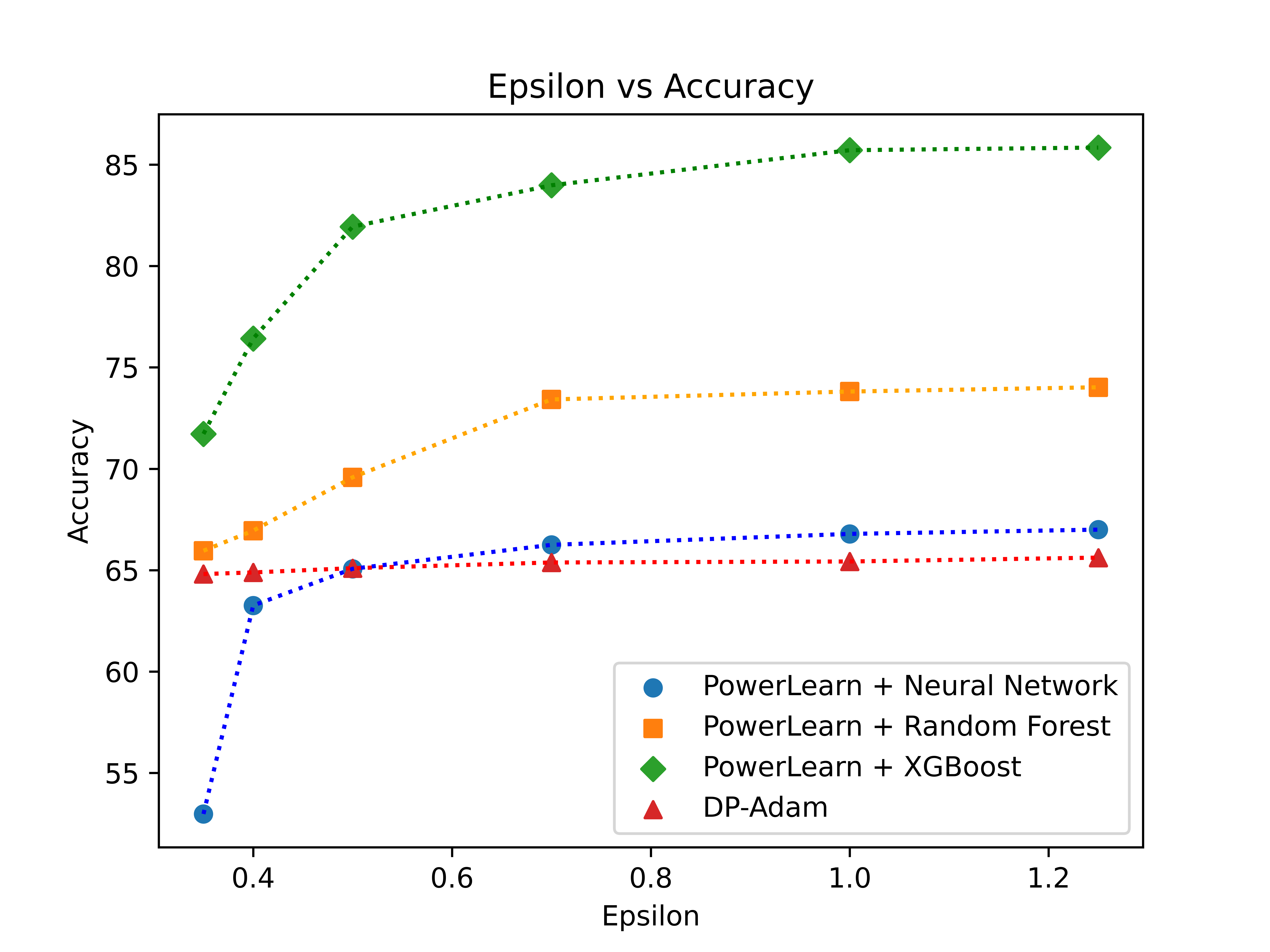}
    \caption{Epsilon vs Accuracy for Forest Cover Dataset shows that we match the performance of the DP-ADAM based neural network, while we show an increase in client resource efficiency for our method in Table \ref{tab:comp_cost}. }
    \label{fig:eps_vs_acc}
\end{figure}

\subsection{Experiment 2: Resource efficiency} \label{subsec:expt_res} We report the computational cost incurred by the client and the server and compare it with the baseline non-private approach along with DP-ADAM. Since DP-Adam does not generate private embeddings, the model needs to be trained entirely on the client side. Our results are summarized in Table \ref{tab:comp_cost}. We use the product of GPU memory requirement and number of steps to reach a particular accuracy, as a proxy to measure the computational cost to attain a certain accuracy. For example, to reach 65\% accuracy with $\epsilon=0.5$ on the Forest cover dataset, DP-Adam requires 7252 units of compute, whereas PowerLearn small requires only 694.7 units. We note that our method is successful in offloading a portion (typically a majority) of the workload to the server, while not losing out by a lot to the baseline non-private approach for faster convergence. We also note that we attain much better balancing of client-server resource efficiencies while also incurring a lesser overall computation cost against private weight release approaches.

\begin{table*}[!htbp]
    \centering
    \begin{tabular}{|c|c|c|c|c|c|} \hline 
         $\epsilon$&  \textbf{Method}&  \textbf{Accuracy}&  \textbf{Client Cost}&  \textbf{Server Cost}& \textbf{Total Cost}
\\ \hline 
         0.5&  DP-ADAM&  65&  7525&  0& 7525
\\ \hline 
         0.5&  PowerSmall&  65&  68.4&  626.3& 694.7
\\ \hline 
         0.5&  PowerLarge&  65&  155.2&  529.5& 684.7
\\ \hline 
         1&  DP-ADAM&  65&  3295.8&  0& 3295.8
\\ \hline 
         1&  PowerSmall&  65&  68.4&  547.9& 616.3
\\ \hline 
         1&  PowerLarge&  65&  155.2&  349.8& 505
\\ \hline 
         1&  PowerLarge&  66&  155.2&  1356.6& 1511.8
\\ \hline 
     $\infty$   &  NonPriv&  66&  54&  1027.9& 1081.9\\ \hline
    \end{tabular}
  \caption{Distribution of compute: Forest Cover (top table), Higgs Boson (middle table) and Adult Income Datasets (bottom table).\\} 
    \label{tab:}
    \centering
    \begin{tabular}{|c|c|c|c|c|c|} \hline 
         $\epsilon$&  \textbf{Method}&  \textbf{Accuracy}&  \textbf{Client Cost}&  \textbf{Server Cost}& \textbf{Total Cost}
\\ \hline 
         0.5&  DP-ADAM&  69&  5134.1&  0& 5134.1
\\ \hline 
         0.5&  PowerLearn&  69&  90.2&  883.7& 973.9
\\ \hline 
         1&  DP-ADAM&  70&  3295.8&  0& 3295.8
\\ \hline 
         1&  PowerLearn&  70&  90.2&  679.6& 769.8
\\ \hline 
         1&  PowerLearn&  73&  90.2&  1125.3& 1215.5
\\ \hline 
         $\infty$ &  NonPriv&  73&  49.2&  128.7& 177.9\\ \hline
    \end{tabular}
    \label{tab:comp_cost}
    \centering
    \begin{tabular}{|c|c|c|c|c|c|} \hline 
         $\epsilon$&  \textbf{Method}&  \textbf{Accuracy}&  \textbf{Client Cost}&  \textbf{Server Cost}& \textbf{Total Cost}
\\ \hline 
         1.5&  DP-ADAM&  80&  131742&  0& 131742
\\ \hline 
         1.5&  PowerLearn&  80&  29.8&  377.14& 406.94
\\ \hline 
         1.5&  PowerLearn&  82&  29.8&  13890& 13919.8
\\ \hline 
         3&  DP-ADAM&  82&  107200&  0& 107200
\\ \hline 
         3&  PowerLearn&  82&  29.8&  1077.79& 1107.59
\\ \hline 
         $\infty$&  NonPriv&  82&  16.8&  672.97& 689.77\\ \hline
    \end{tabular}
    \label{tab:comp_cost}
\end{table*}

\subsection{Experiment 3: Performance of proposed defense against attacks} We study the empirical privacy leakage of our embedding and compare them against the embeddings released by non-private and DP-ADAM trained  models. Feature space hijacking \cite{pasquini2021unleashing} is a popular attack on embeddings and has been successful in reconstructing training data. We try to simulate this attack by assuming a malicious server, with access to public data points which follow similar data distribution as the training data used. To gauge the success of the attack, we evaluate the percentage of samples for which, the server was able to reconstruct the categorical feature from the embeddings. We show that PowerLearn has a leakage on only $0.36\%$ of samples while DP-ADAM has a substantial leakage on $4.5\%$ of the samples. The results are summarized in \ref{tab:attack_tab}
\begin{table}[h!]
\centering
    \begin{tabular}{|c|c|c|} \hline 
         \textbf{Model}& \textbf{Accuracy} &  \textbf{MSE} \\ \hline 

         Non Private&  3.63 \%&  0.0008\\ \hline 
         DP-ADAM&  4.5 \%&  0.0019\\ \hline 
         PowerLearn&  0.36 \%&  0.2416\\ \hline
    \end{tabular}
      \caption{Comparison of the defenses on a popular reconstruction attack applicable to our setting called the feature space hijacking attack (FSHA).}
    \label{tab:attack_tab}%
\end{table}
\subsection{Experiment 4: Lipschitz privacy loss evaluation} We evaluate the theoretical privacy leakage, using our lipschitz loss on the embeddings and compare it to embeddings generated without using lipschitz privacy loss term (non-private baseline) and upon using DP-ADAM to train the embeddings. Our results are summarized in the four Figures in \ref{fig:priv_eps} and \ref{fig:test_hist} . We show that PowerLearn achieves a higher privacy level than DP-ADAM over the activations. It is also worth noting that DP-ADAM is a method to provide a chosen level of privacy through the model weights. DP-ADAM does not provide any theoretical privacy guarantee on the activations. Thereby, this further showcases the gap over existing methods such as (DP-SGD, DP-ADAM or DP-FTRL) that our method is filling in on for private activation release. We also see that the histograms do not vary much when the learnt privatization network is applied on train and test sets in order to release the corresponding private activations. This empirically showcases a good generalization of the privatization capability of our approach.
\begin{figure*}[!htbp]
    \centering
    \begin{floatrow}
    \includegraphics[scale=0.4]{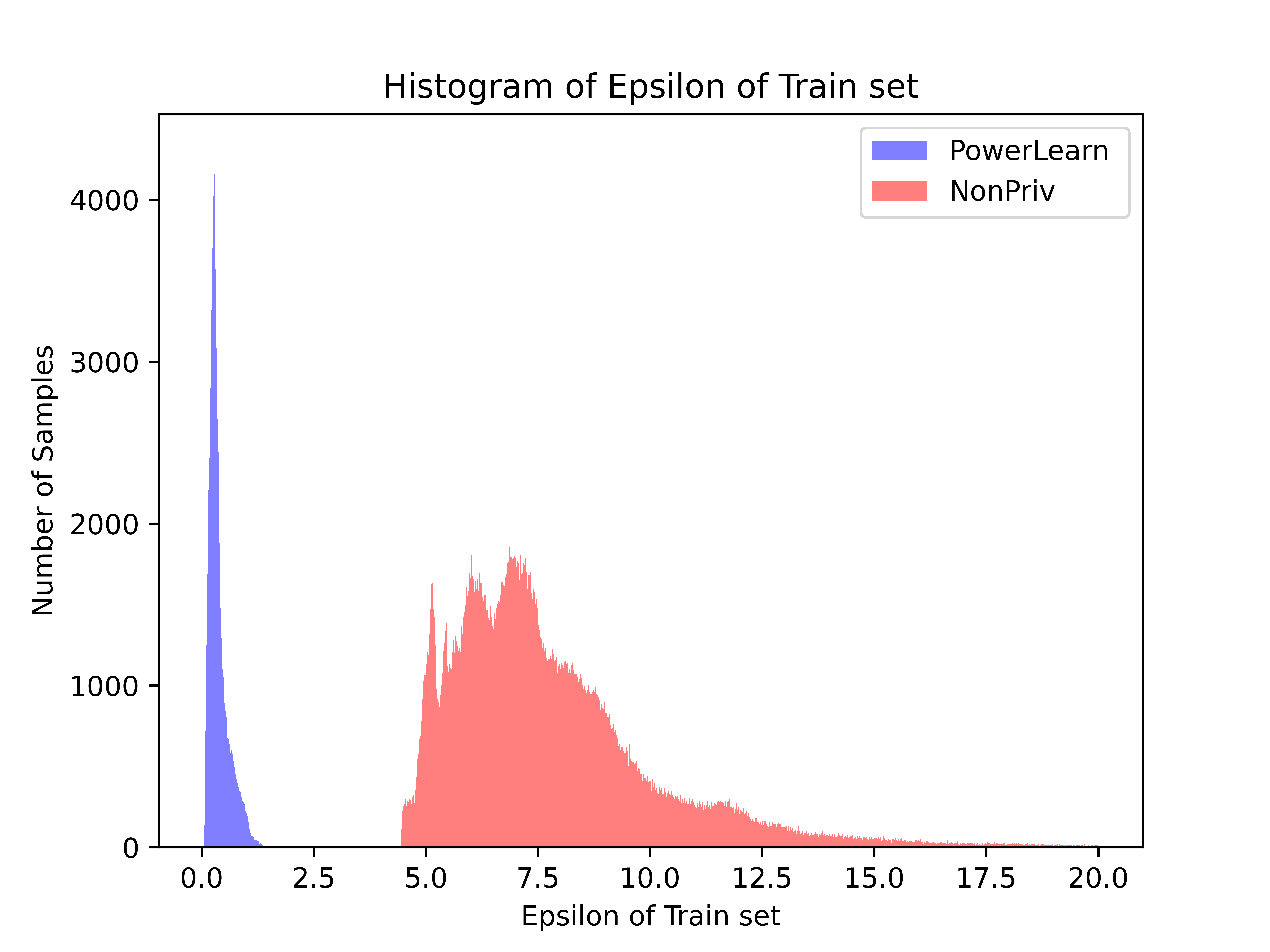}
      \caption{Comparison of histograms of $\epsilon$ between PowerLearn and baseline approaches}
    \label{fig:enter-label1}
    \includegraphics[scale=0.4]{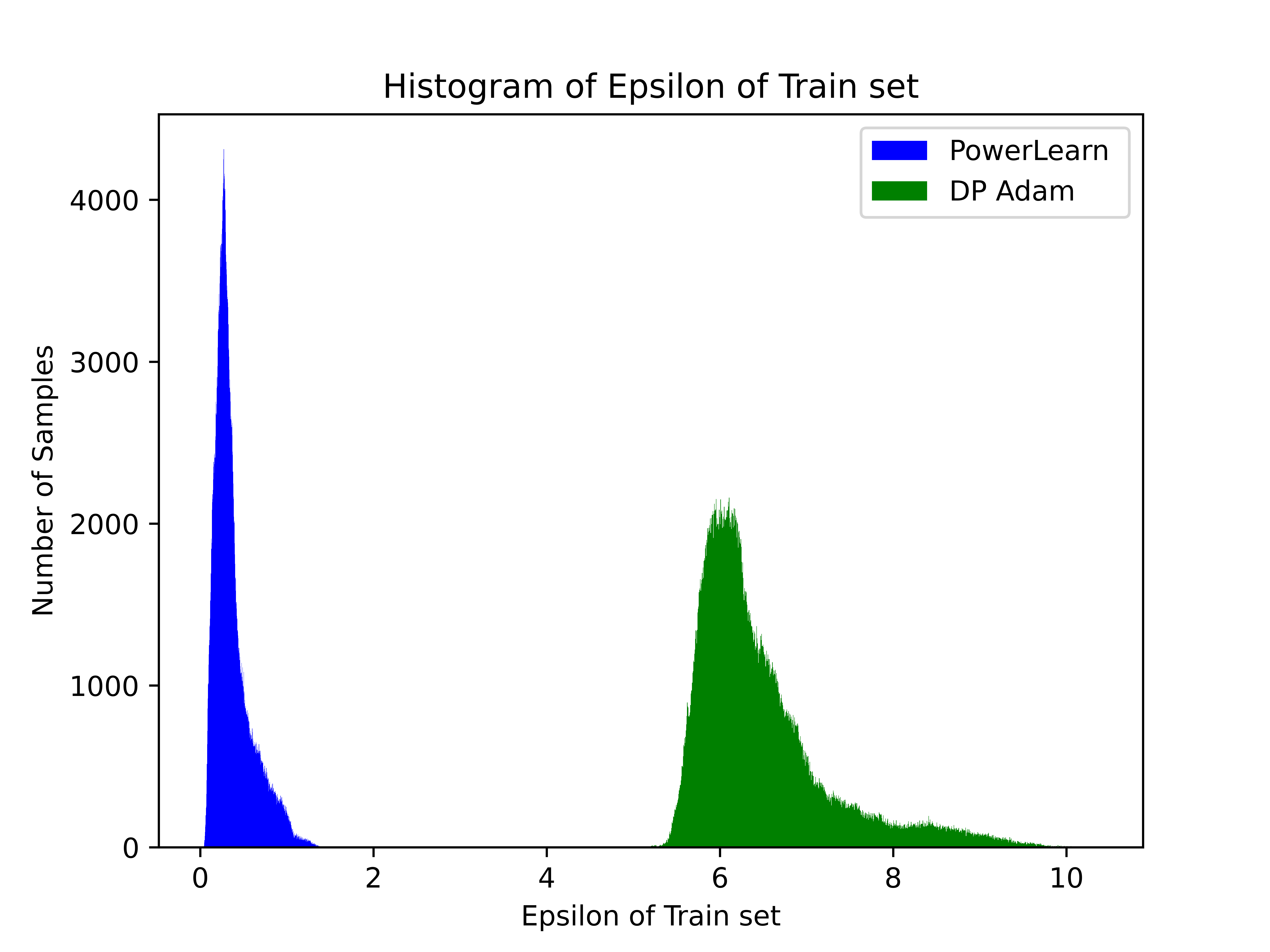}
    \label{fig:priv_eps}
    \end{floatrow}
\end{figure*}

\subsection{Experiment 5: Choice of power.} We use p=1 for most of our experiments, as the client model is smallest in this case. In this part, we try to analyze the effects of varying p and see the convergence of the privacy loss. Note that while the network depth increases, the number of parameters to train remains the same as we still use the same $H$ for multiple power iterations. The results of our experiments on Higgs Boson Dataset are summarized in Figure \ref{fig:p_vs_priv} where the privatization performance is shown to improve with increasing $p$.



\begin{figure}[h!]
    \centering
    \includegraphics[width=\linewidth]{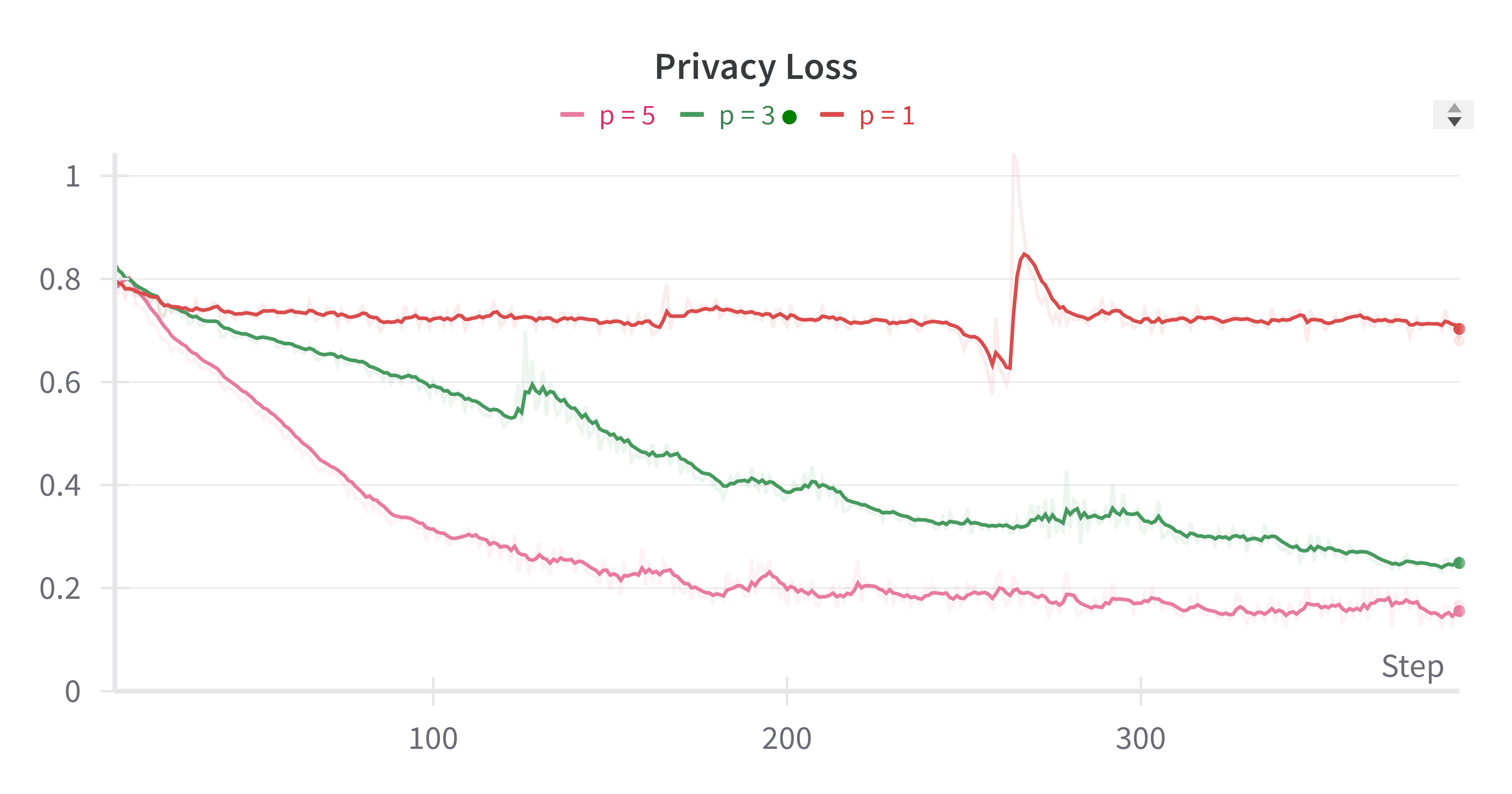}
    \caption{Choice of power $p$ vs. Privacy loss}
    \label{fig:p_vs_priv}
\end{figure}

\subsection{Experiment 6: Convergence of utility, privacy and joint losses} In this additional experiment, we evaluate how the  client model tries to minimize the proposed privacy loss and utility loss jointly. The difference in the two models is the batch size used in training. As we can observe in Figure \ref{fig:priv_loss_conv} in the Appendix that the PowerLearn Large model is able to perform better on both the losses, while incurring more computational cost on the client. When we move to the server models, as shown in Figure \ref{fig:priv_acc_conv}, we observe the convergence of accuracies of the PowerLearn models. The figure  helps explain the computational expense advantage of PowerLearn over DP-ADAM, which reaches competitive accuracies slower. Additionally, we notice that while test accuracy varies by changing $\epsilon$ , the convergence is still at the same rate for all the PowerLearn models.

\section{Additional Experimental Details} \label{experimentsDetails}
\subsubsection{Datasets}\label{appendix:dataset}
We detail the datasets which were used in Section \ref{sec:exp} and summarize them in Table \ref{table:datasets}. All of the three datasets are licensed under a Creative Commons Attribution 4.0 International (CC-BY 4.0) license.
\begin{itemize}
    \item \textbf{Adult Income.} The Adult Income Dataset \cite{income_dataset}, also known as the Census Dataset, contains information extracted by Barry Becker from the 1994 Census database. A set of reasonably clean records was obtained using basic filters. The dataset aims to predict whether an individual's annual income exceeds \$50,000, based on factors such as education level, age, gender, and occupation. It includes over 48,000 samples and is divided into two classes. 
    \item \textbf{Forest Cover.} The Forest Cover Dataset \cite{covertype_dataset} is used to predict forest cover type based on cartographic variables alone, without remotely sensed data. Each observation represents a 30 x 30 meter cell, with forest cover type determined from US Forest Service (USFS) Region 2 Resource Information System (RIS) data. The independent variables, derived from US Geological Survey (USGS) and USFS data, include both continuous attributes like elevation, aspect, and slope, and binary variables for qualitative data such as wilderness areas and soil types, encoded as 0 or 1. The dataset is unscaled and comprises approximately 580,000 samples classified into seven forest cover types. The study area covers four wilderness areas in the Roosevelt National Forest of northern Colorado, where minimal human disturbances allow forest cover types to reflect natural ecological processes. 
    \item \textbf{Higgs Boson.} The Higgs Boson Dataset \cite{higgs_dataset} is generated using Monte Carlo simulations. The first 21 features represent kinematic properties measured by particle detectors in the accelerator, while the last seven features are high-level functions derived from the first 21, designed by physicists to aid in class discrimination. The dataset contains 240,000 training samples, which we used for all our experiments, and consists of two classes.
\end{itemize}

\begin{table}[H]
\centering
\caption{
  Summary of the datasets and tasks used in our empirical setup.
  }
\label{table:datasets} 
\vspace{0.5em}
\begin{tabular}{@{}lccc@{}}
\toprule
\textbf{Dataset} & \textbf{\# Samples}& \multicolumn{1}{l}{\textbf{\# Features}} & \multicolumn{1}{l}{\textbf{\# Classes}} \\ \midrule
Forest Cover  &580k& 54& 7\\
Higgs Boson  &240k& 30& 2\\
Adult Income  &48k& 14& 2\\
\end{tabular}
\end{table}

\subsection{Experimental Settings}\label{appendix:moredetails}

In all experiments, we use an $80\%-20\%$  split for the dataset. Initially, the training data points are divided, and embeddings are created using the client model, followed by applying the same split on the server model. The client model consists of two networks: the first network learns the $H$ matrix for each data point using a neural network, while the second network minimizes the utility cost of the embeddings. To measure the client-server work split, the server model is consistently a neural network. Additionally, we use an XGBoost classifier and a random forest classifier on the server-side embeddings to study the privacy-utility trade-off. The DP-ADAM model always has the same architecture as the server neural network. Our non-private baseline maintains the same architecture for both the server and client models. The training configuartion for all the datasets and models is given in Tables \ref{tab:clie_mod} and \ref{tab:serv_mod}.

For each dataset in the privacy-utility tradeoff experiment, we report both the average test accuracy/MSE and its corresponding one standard error based on multiple runs.  We measure the client and server cost using the product of GPU RAM utilization and number of steps to reach a particular accuracy. The unit of measurement in all the tables across datasets is terabytes (steps).  

Finally, in the experiment about defence against feature hijack attack, we first learn a function to map the embeddings to the original points using an autoencoder. We then find the indices with maximum value in the decoder output and assign the one hot encodings of the categorical features. We use this reconstructed vector and measure it's similarity to the private data. In Table \ref{tab:attack_tab} ,  we measure the number of points for which the attack is successful in getting the categorical features and report the accuracy. The mean squared error and correlation coefficient are calculated between all the corresponding reconstructed and original private data points.

\begin{table*}[!htbp]
    \centering
    \begin{tabular}{|c|c|c|c|c|}
        \hline
        \textbf{Dataset} & \textbf{Model Name} & \textbf{\#Steps} & \textbf{Batch size} & \textbf{Learning rate} \\ \hline
        \multirow{4}{*}{Forest Cover Type} & PowerLearn Small & 100 & 128 & 0.0003 \\ \cline{2-5}
        & PowerLearn Large & 100 & 512 & 0.0003 \\ \cline{2-5}
        & DP-ADAM & 10k & 4096 & 0.0003 \\ \cline{2-5}
        & Non Private & 100 & 128 & 0.0003 \\ \hline
        \multirow{3}{*}{Higgs Boson} & PowerLearn Large & 100 & 512 & 0.0003 \\ \cline{2-5}
        & DP-ADAM & 10k & 4096 & 0.001 \\ \cline{2-5}
        & Non Private & 100 & 512 & 0.0003 \\ \hline
        \multirow{3}{*}{Adult Income} & PowerLearn Large & 30 & 128 & 0.001 \\ \cline{2-5}
        & DP-ADAM & 100k & 512 & 0.001 \\ \cline{2-5}
        & Non Private & 30 & 128 & 0.001 \\ \hline
    \end{tabular}
    \caption{Client Model Details}
    \label{tab:clie_mod}
\end{table*}

\begin{table*}[!htbp]
    \centering
    \begin{tabular}{|c|c|c|c|c|}
        \hline
        \textbf{Dataset} & \textbf{Model Name} & \textbf{\#Steps} & \textbf{Batch size} & \textbf{Learning rate} \\ \hline
        \multirow{3}{*}{Forest Cover Type} & PowerLearn Small& 10k& 4096& 0.0003
\\ \cline{2-5}
        & PowerLearn Large& 10k& 4096& 0.0003
\\ \cline{2-5}
        & Non Private& 10k& 4096& 0.0003
\\ \hline
        \multirow{2}{*}{Higgs Boson} & PowerLearn Large& 10k& 4096& 0.0003
\\ \cline{2-5}
        & Non Private& 10k& 4096& 0.0003
\\ \hline
        \multirow{2}{*}{Adult Income} & PowerLearn Large& 15k& 512& 0.001
\\ \cline{2-5}
        & Non Private& 15k& 512& 0.001\\ \hline
    \end{tabular}
    \caption{Server Model Details}
    \label{tab:serv_mod}
\end{table*}

\subsection{Hardware \& Code} \label{harwareCode}
Our experiments were carried out on a single  NVIDIA A100-SXM4-80GB GPU. The algorithms are implemented in Python using PyTorch \cite{paszke2019pytorch}. The code is available at 
\url{https://anonymous.4open.science/r/Power-Mechanism-new-submit-6039/}{https://anonymous.4open.science/r/Power-Mechanism-new-submit-6039/}

\section{Conclusion} The proposed {\ours} fills in the gap in the literature on differential privacy preserving schemes for release of activations from a neural network. Current works instead deal with differentially private release of model weights. We extensively evaluate and show the benefits of our holistic approach based on a co-design of distributed and private machine learning aspects of the problem. We show substantial improvements in the privacy-utility trade-offs and resource efficiencies of our method in comparison to several baselines.

\section{Limitations and Future Work}\label{app_limit} Although our method theoretically applies in principle to many classes of learnable data transformations that induce privacy, we have solely focused in this work on tabular datasets in terms of our experimental setups and evaluations. Applying our approach to other modalities such as speech, text and vision will require some more co-design between the architectural aspects of the models and the theoretical aspects of inducing privacy on objects such as word or sentence embeddings for example, in a semantically meaningful way at the same time. This is out of scope for the main focus of this paper, and this would be the main focus of our future works. That said, our setup applied to tabular data itself results in several important applications for privacy-preserving collaborative learning, as many datasets in the real-world are tabular.



\bibliography{tmlr}
\bibliographystyle{tmlr}

\appendix
\section{Proofs}
\subsection{Proof of Theorem \ref{equiTheorem}: Equivalence of privacy} \label{equiProof}
\begin{proof}
Fix \( x, x' \in \mathbb{R}^d \), and let \( z \in \mathcal{Z} \) be arbitrary. Define the scalar function
\[
\phi(t) := \ln g(x + t(x' - x), z), \quad t \in [0,1].
\]
Then \( \phi \) is differentiable, since \( \ln g(\cdot, z) \) is differentiable.

By the Mean Value Theorem, there exists \( t_0 \in (0,1) \) such that
\[
\phi(1) - \phi(0) = \phi'(t_0).
\]
Thus,
\[
\ln g(x', z) - \ln g(x, z) = \phi'(t_0).
\]

Using the chain rule,
\[
\phi'(t_0) = \left\langle \nabla_x \ln g(x + t_0(x' - x), z), x' - x \right\rangle.
\]

Applying the Cauchy--Schwarz inequality for the Euclidean inner product,
\[
|\ln g(x', z) - \ln g(x, z)| \leq \left\| \nabla_x \ln g(x + t_0(x' - x), z) \right\|_2 \cdot \|x' - x\|_2.
\]

By the assumption that \( \| \nabla_x \ln g(x, z) \|_2 \leq \varepsilon \) for all \( x \), we conclude
\[
|\ln g(x', z) - \ln g(x, z)| \leq \varepsilon \|x' - x\|_2.
\]

This proves that \( \ln g(\cdot, z) \) is \( \varepsilon \)-Lipschitz with respect to the Euclidean norm, and hence that the mechanism satisfies Lipschitz privacy.
\end{proof}

\section{Derivation of privacy inducing loss}
\subsubsection{Privacy proof}
\begin{proof}\label{powProof}
The equation $\mathbf{z} = G(\mathbf{x})$ can be unrolled as
\begin{equation}
    \mathbf{z} = g_{p-1} \circ  g_{p-2} \circ \dots \circ g_0 (\mathbf{x})
\end{equation} where $g_{p-1} \circ  g_{p-2}(\cdot)=\mathbf{H}(g_{p-2}(\cdot)).g_{p-2}(\cdot)$. If $g_k$ is a one-to-one function on the support of $\mathbf{X}$ whose pdf is given by $f_{\mathcal{X}}(x)$ where $ x \in \mathbb{R}^k$, then the pdf of $\mathbf{\mathbf{z} = G(\mathbf{x})}$ is  
$$ h_{\mathcal{Z}}(\mathbf{z}) = f_{\mathcal{X}}(G^{-1}(\mathbf{z})) |\det(\mathbf{J}(G^{-1}(\mathbf{z})))|$$ 
for $\mathbf{z}$ in the range of $G$, where $\mathbf{J(x)}$ is the Jacobian matrix of $G$ that is evaluated at $\mathbf{x}$. This is classically known as the multidimensional change of variable theorem in the context of probability density functions.  
But since we have $g_{p-1} \circ g_{p-2} \circ \dots \circ g_0 (\mathbf{x})$ instead of a single $G(\cdot)$, this can be written as 
$$h_{\mathcal{Z}}(\mathbf{z}) 
=  h_{\mathcal{W}_{p-1}}(g_{p-1}^{-1}(\mathbf{z})) \left\vert \det\frac{\partial \mathbf{w}_{p-1} }{\partial\mathbf{z} } \right\vert $$ 
which is the same as the following.
$$h_{\mathcal{Z}}(\mathbf{z}) 
=  h_{\mathcal{W}_{p-1}}(\mathbf{w}_{p-1})\left\vert \det\frac{\partial \mathbf{w}_{p-1} }{\partial\mathbf{z} } \right\vert $$ Upon applying a logarithm, we get the following.
$$ \log(h_{\mathcal{Z}}(\mathbf{z}))
=  \log(h_{\mathcal{W}_{p-1}}(\mathbf{w}_{p-1})) + \log(\left\vert \det\frac{\partial \mathbf{w}_{p-1} }{\partial\mathbf{z} } \right\vert) $$ After writing the last terms in terms of a reciprocal, we have the following.
$$ \log(h_{\mathcal{Z}}(\mathbf{z}))
=  \log(h_{\mathcal{W}_{p-1}}(\mathbf{w}_{p-1})) - \log(\left\vert \det\frac{ \partial\mathbf{z}}{\partial \mathbf{w}_{p-1}} \right\vert) $$ Now writing this in terms of the recursive composition that goes into generating $\mathbf{z}$, we get the following.
$$
\begin{alignedat}{2}
&\log(h_{\mathcal{Z}}(\mathbf{z})) = \log(h_{\mathcal{W}{p-2}}(\mathbf{w}{p-2})) \\
&\quad - \log(\left\vert \det\frac{ \partial\mathbf{w}{p-1}}{\partial \mathbf{w}{p-2}} \right\vert) \
&\quad - \log(\left\vert \det\frac{ \partial\mathbf{z}}{\partial \mathbf{w}_{p-1}} \right\vert)
\end{alignedat}
$$ 
Writing this in terms of a summation of Jacobians, we get the following.
$$\log(h_{\mathcal{Z}}(\mathbf{z}) 
= \log(f_{\mathcal{X}}(\mathbf{x})) - \sum_{k=0}^{p-1}\log \left(\left\vert \det \mathbf{J}_k \right\vert \right) $$ 
Now expressing this equation in terms of the condition needed towards Lipschitz privacy we get the following.
$$   \frac{\partial}{\partial \mathbf{x}}\log h_{\mathcal{Z}}(\mathbf{z})=  \frac{\partial}{\partial \mathbf{x}} \log f_{\mathcal{X}}(\mathbf{x}) -  \frac{\partial}{\partial \mathbf{x}} \sum_{k=0}^{p-1}  \log(|\det(\mathbf{J}_k)) $$
As the final Lipschitz privacy condition is based on a norm, we re-express it accordingly as follows which is our proposed condition that forms the privacy-inducing loss for our privatization network.
$$  \left\lVert \frac{\partial}{\partial \mathbf{x}}\log h_{\mathcal{Z}}(\mathbf{z}) \right \rVert = \left\lVert 
\frac{1}{f_{\mathcal{X}}(\mathbf{x})} \frac{\partial f_{\mathcal{X}}(\mathbf{x})}{\partial \mathbf{x}} -  \frac{\partial}{\partial \mathbf{x}} \sum_{k=0}^{p-1}  \log(|\det(\mathbf{J}_k))\right \rVert \leq \epsilon$$
\end{proof}
\section{Calibrating $\delta$ of Differential Privacy}

Since we have a high probability but approximate bound on the interval of the true density function, we have to account for the probability with which the privacy leaks. Let event $E$ be the event that the true probability lies within the confidence interval with high probability $ 1 - \alpha$. Now, our mechanism acting on input $M(x)$ can behave under two cases. In one case, it obeys differential privacy (denoted by event $T$). Then by the \textit{Law of Total Probability}, we have

$$ \probP[M(x) \in T] = \probP[M(x) \in T | E]\probP[E] + \probP[M(x) \in T | !E]\probP[!E] $$ 
Now the probability of event $E$ occuring is $\alpha$ and $1-\alpha$ otherwise. Therefore we have,

$$ \probP[M(x) \in T] = \probP[M(x) \in T | E](1 - \alpha) +  \probP[M(x) \in T | !E]\alpha $$  This simplifies to be
$$
\begin{alignedat}{2}
&\probP[M(x) \in T] = \probP[M(x) \in T | E] + \\
&\quad \alpha(\probP[M(x) \in T | !E] - \probP[M(x) \in T | E])
\end{alignedat}
$$
$$ \therefore \probP[M(x) \in T] \leq \probP[M(x) \in T | E] +  \alpha  $$ 
Now using the definition of $\epsilon - \delta $ Differential Privacy we know
$$  \probP[M(x) \in T | E] \leq  e^{\epsilon} \probP[M(x') \in T | E] + \delta  $$ Now as we are so far operating with a $\delta = 0$, we therefore have 
$$ \probP[M(x) \in T] \leq e^{\epsilon}\probP[M(x') \in T | E] + \delta  +  \alpha  $$ 
which gives us,

$$ \probP[M(x) \in T] \leq e^{\epsilon}\probP[M(x') \in T | E] +  \alpha  $$ 

\section{Supporting lemmas for the reconstruction lower-bound}
\begin{proof}
We begin by recalling that the squared Euclidean norm of a vector $v \in \mathbb{R}^d$ is defined as
\[
\|v\|^2 = \sum_{i=1}^d v_i^2.
\]
Applying this to the random vector $A(z) - \mu(x)$, we write:
\[
\|A(z) - \mu(x)\|^2 = \sum_{i=1}^d (A_i(z) - \mu_i(x))^2.
\]
Now take expectation over $z \sim p_Z(\cdot \mid x)$:
\begin{align*}
\mathbb{E}_{z}[\|A(z) - \mu(x)\|^2] &= \mathbb{E}_{z}\left[\sum_{i=1}^d (A_i(z) - \mu_i(x))^2\right] \\
&= \sum_{i=1}^d \mathbb{E}_{z}\left[(A_i(z) - \mu_i(x))^2\right] = \sum_{i=1}^d \operatorname{Var}[A_i(z) \mid x].
\end{align*}
By the definition of the conditional covariance matrix,
\[
\operatorname{Cov}(A(z) \mid x) = \mathbb{E}_z[(A(z) - \mu(x))(A(z) - \mu(x))^T],
\]
which is a $d \times d$ matrix whose $(i,j)$ entry is
\[
\operatorname{Cov}_{ij}(A(z) \mid x) = \mathbb{E}[(A_i(z) - \mu_i(x))(A_j(z) - \mu_j(x))].
\]
Hence, the trace of the covariance matrix is
\[
\operatorname{Tr}(\operatorname{Cov}(A(z) \mid x)) = \sum_{i=1}^d \operatorname{Cov}_{ii}(A(z) \mid x) = \sum_{i=1}^d \operatorname{Var}[A_i(z) \mid x].
\]
Combining both results, we conclude that
\[
\mathbb{E}_z[\|A(z) - \mu(x)\|^2] = \operatorname{Tr}(\operatorname{Cov}(A(z) \mid x)). \qed
\]
\end{proof}
\begin{proof}
We first recall that $\nabla_x \log f_X(x) = \frac{\nabla_x f_X(x)}{f_X(x)}$. Therefore,
\[
\mathbb{E}_{x \sim f_X}[\nabla_x \log f_X(x)] = \int \nabla_x \log f_X(x) f_X(x) dx = \int \nabla_x f_X(x) dx.
\]
Now apply the divergence theorem over $\mathbb{R}^d$, assuming sufficient decay of $f_X(x)$ and its gradient:
\[
\int_{\mathbb{R}^d} \nabla_x f_X(x) dx = \lim_{R \to \infty} \int_{B_R(0)} \nabla_x f_X(x) dx = \lim_{R \to \infty} \int_{\partial B_R(0)} f_X(x) \, dS(x) = 0.
\]
This holds if $f_X(x)$ decays to zero faster than any polynomial. Therefore,
\[
\mathbb{E}_{x \sim f_X}[\nabla_x \log f_X(x)] = 0. \qed
\]
\end{proof}
\section{Lower bound on reconstruction error}
\begin{proof}
Using the decomposition from the law of total expectation and Lemma 1
\[
R(A) = \mathbb{E}_x \left[ \operatorname{Tr}(\operatorname{Cov}(A(z) \mid x)) + \|\mu(x) - x\|^2 \right].
\]

Upon applying the van Trees inequality (which holds under regularity and Lemma 2), we get
\[
\operatorname{Cov}(A(Z)) \succeq J_\mu(x)(\mathcal{I}_{Z|X}(x) + \mathcal{I}(f_X))^{-1} J_\mu(x)^T.
\]

Taking the trace of both sides yields
\[
\operatorname{Tr}(\operatorname{Cov}(A(z) \mid x)) \geq \operatorname{Tr}(J_\mu(x)(\mathcal{I}_{Z|X}(x) + \mathcal{I}(f_X))^{-1} J_\mu(x)^T).
\]
Substitute this back into the expression for \( R(A) \)
\[
R(A) \geq \mathbb{E}_x[\operatorname{Tr}(J_\mu(x)(\mathcal{I}_{Z|X}(x) + \mathcal{I}(f_X))^{-1} J_\mu(x)^T) + \|\mu(x) - x\|^2].
\]

Now assume \( \mu(x) = x \), so that \( J_\mu(x) = I_d \). Then
\[
R(A) \geq \operatorname{Tr}((\mathcal{I}_{Z|X}(x) + \mathcal{I}(f_X))^{-1}).
\]

By Lemma 3, and our assumption that the nullspaces of \( \mathcal{I}_{Z|X}(x) \) and \( \mathcal{I}(f_X) \) intersect trivially, the matrix \( M = \mathcal{I}_{Z|X}(x) + \mathcal{I}(f_X) \) is symmetric positive definite.

Let \( \lambda_1, \dots, \lambda_d > 0 \) denote the eigenvalues of \( M \). By Jensen’s inequality for convex functions applied to the eigenvalues we get
\[
\operatorname{Tr}(M^{-1}) = \sum_{i=1}^d \frac{1}{\lambda_i} \geq \frac{d^2}{\sum_{i=1}^d \lambda_i} = \frac{d^2}{\operatorname{Tr}(M)}.
\]

By assumption of Lipschitz privacy,
\[
\|\nabla_x \log p_Z(z \mid x)\|^2 \leq \varepsilon^2 \Rightarrow \operatorname{Tr}(\mathcal{I}_{Z|X}(x)) \leq \varepsilon^2.
\]

Therefore we have,
\[
\operatorname{Tr}(M) \leq \varepsilon^2 + \operatorname{Tr}(\mathcal{I}(f_X))
\Rightarrow R(A) \geq \frac{d^2}{\varepsilon^2 + \operatorname{Tr}(\mathcal{I}(f_X))}.
\]
\qed
\end{proof}
\begin{remark}
The assumption that the null spaces of \( \mathcal{I}_{Z|X}(x) \) and \( \mathcal{I}(f_X) \) intersect only at zero is essential. Without it, the matrix \( M = \mathcal{I}_{Z|X}(x) + \mathcal{I}(f_X) \) may be singular, and its inverse, as required in the theorem statement, would not exist. This is not merely a technicality, but a fundamental requirement to ensure that the van Trees inequality yields a meaningful finite lower bound. This assumption is often mild in practice. Specifically, if the prior Fisher information matrix \( \mathcal{I}(f_X) \) is strictly positive definite, i.e., \( \mathcal{I}(f_X) \succ 0 \), then its null space is trivial. This holds for any prior with full support and differentiable density, such as multivariate Gaussians or Laplace distributions. In this case, \( \operatorname{null}(\mathcal{I}(f_X)) = \{0\} \), and so the intersection with any other null space is automatically trivial. 

Thus, the assumption holds generically unless both the prior and the mechanism are degenerate in the same direction. When this degeneracy does occur, reconstruction is impossible in that direction, and the bound degenerates as expected.
\end{remark}

\begin{proof}
Let $f(x)$ be the true density, and let $\widehat{f}(x)$ be the kernel density estimator constructed from the samples $x_1, \dots, x_n$. Denote by $\nabla f(x)$ and $\nabla \widehat{f}(x)$ their gradients. Then the score function is given by $s(x) = \nabla f(x)/f(x)$ and the estimated score is $\widehat{s}(x) = \nabla \widehat{f}(x)/\widehat{f}(x)$.
Define the errors,
\[
\delta(x) := \widehat{f}(x) - f(x), \quad \delta^{(1)}(x) := \nabla \widehat{f}(x) - \nabla f(x).
\]
We want to estimate the deviation between the estimated score and the true score as below
\[
\widehat{s}(x) - s(x) = \frac{\nabla \widehat{f}(x)}{\widehat{f}(x)} - \frac{\nabla f(x)}{f(x)}.
\]
Using the identity for the difference of ratios,
\[
\frac{a + \delta a}{b + \delta b} - \frac{a}{b} \approx \frac{\delta a}{b} - \frac{a \delta b}{b^2},
\]
we get,
\[
\widehat{s}(x) - s(x) \approx \frac{\delta^{(1)}(x)}{f(x)} - \frac{\nabla f(x) \delta(x)}{f(x)^2}.
\]
Taking squared norms and expectations,
\[
\mathbb{E} \left[ \|\widehat{s}(x) - s(x)\|^2 \right] \leq 2 \mathbb{E} \left[ \left\| \frac{\delta^{(1)}(x)}{f(x)} \right\|^2 \right] + 2 \mathbb{E} \left[ \left\| \frac{\nabla f(x) \delta(x)}{f(x)^2} \right\|^2 \right].
\]
From standard KDE theory,
\[
\mathbb{E} \left[ \|\delta^{(1)}(x)\|^2 \right] = O\left(\frac{1}{n h^{d+4}}\right), \quad \mathbb{E} \left[ \delta(x)^2 \right] = O\left(\frac{1}{n h^d}\right).
\]
Hence the leading term in estimating the squared error in the score is,
\[
\mathbb{E} \left[ \|\widehat{s}(x) - s(x)\|^2 \right] = O\left( \frac{1}{n h^{d+4}} \right).
\]
This implies that the outer product $\widehat{s}(x)\widehat{s}(x)^\top$ differs from $s(x)s(x)^\top$ by a matrix with entries that deviate by $O(1 / (n h^{d+4}))$ in expectation. Averaging these over $n$ samples yields the deviation in the trace of the Fisher information estimate,
\[
\left|\operatorname{Tr}(\widehat{\mathcal{I}}(f_X)) - \operatorname{Tr}(\mathcal{I}(f_X))\right| = O\left( \frac{1}{n h^{d+4}} \right).
\]
We denote the constant factor in this bound by $c_1^2$ and substituting this deviation into the lower bound from Theorem 3, we get
\[
\mathcal{R}(A) \geq \frac{d^2}{\varepsilon^2 + \operatorname{Tr}(\widehat{\mathcal{I}}(f_X)) + \frac{c_1^2}{n h^{d+4}}}.
\]
\qed
\end{proof}
\end{document}